%% file: techreport.tex
\documentclass[preprint]{imsart}

\RequirePackage[OT1]{fontenc}
\RequirePackage{amsthm,amsmath,amssymb}
\RequirePackage[round]{natbib}
\RequirePackage[colorlinks,citecolor=blue,urlcolor=blue]{hyperref}
\RequirePackage{fullpage}
\RequirePackage{todonotes}
\RequirePackage{bm}
\RequirePackage{booktabs}
\RequirePackage{comment}
\RequirePackage{subcaption}
\RequirePackage{algorithm}
\RequirePackage[noend]{algpseudocode}

\usepackage{macros}

\arxiv{arXiv:0000.0000}

\startlocaldefs
\numberwithin{equation}{section}

\newtheorem{lemma}{Lemma}[section]
\newtheorem{proposition}{Proposition}[section]
\newtheorem{theorem}{Theorem}[section]

\endlocaldefs

\begin{document}

\begin{frontmatter}
\title{Estimating the Spectral Density \\ of Large Implicit Matrices}
\runtitle{Estimating the Spectral Density of Large Implicit Matrices}

\begin{aug}
\author{\fnms{Ryan P.} \snm{Adams}%
\thanksref{google}\thanksref{princeton}}, %
\author{\fnms{Jeffrey} \snm{Pennington}%
\thanksref{google}}, %
\author{\fnms{Matthew J.} \snm{Johnson}%
\thanksref{google}}, %
\\ \author{\fnms{Jamie} \snm{Smith}%
\thanksref{google}}, %
\author{\fnms{Yaniv} \snm{Ovadia}%
\thanksref{google}}, %
\author{\fnms{Brian} \snm{Patton}%
\thanksref{google}}, %
\and%
\author{\fnms{James} \snm{Saunderson}%
\thanksref{monash}}%

\runauthor{R.P. Adams et al.}

\address{Google\thanksmark{google}, Monash University\thanksmark{monash} and Princeton University\thanksmark{princeton}}
\end{aug}

\begin{abstract}
\input{abstract.txt}
\end{abstract}

\end{frontmatter}

\input{introduction.tex}

\input{formulation.tex}

\input{power.tex}

\input{traces.tex}

\input{evaluation.tex}
\input{conclusion.tex}

\bibliographystyle{plainnat}
\bibliography{refs}

\end{document}

%% file: abstract.txt
Many important problems are characterized by the eigenvalues of a
large matrix.  For example, the difficulty of many optimization
problems, such as those arising from the fitting of large models in
statistics and machine learning, can be investigated via the spectrum
of the Hessian of the empirical loss function.  Network data can be
understood via the eigenstructure of a graph Laplacian matrix using
spectral graph theory.  Quantum simulations and other many-body
problems are often characterized via the eigenvalues of the solution
space, as are various dynamic systems.  However, na\"{i}ve eigenvalue
estimation is computationally expensive even when the matrix can be
represented; in many of these situations the matrix is so large as to
only be available \emph{implicitly} via products with vectors.  Even
worse, one may only have noisy estimates of such matrix vector
products.  In this work, we combine several different techniques for
randomized estimation and show that it is possible to construct
unbiased estimators to answer a broad class of questions about the
spectra of such implicit matrices, even in the presence of noise.  We
validate these methods on large-scale problems in which graph theory and random matrix theory provide ground truth.

%% file: introduction.tex
\section{Introduction}

The behaviors of scientific phenomena and complex engineered systems are often characterized via matrices.  The entries of these matrices may reflect interactions between variables, edges between vertices in a graph, or the local geometry of a complicated surface.  When the systems under consideration become large, these matrices also become large and it rapidly becomes uninformative --- and potentially infeasible --- to investigate the individual entries.  Instead, we frame our questions about these matrices in terms of aggregate properties that are insensitive to irrelevant details.

To this end, the eigenvalues and eigenvectors of these matrices often provide a succinct way to identify insightful global structure of complex problems.  The eigenvalues of the Hessian matrix in an optimization problem help us understand whether the problem is locally convex or has saddle points.  The influential PageRank algorithm \citep{pagerank} examines the principal eigenvector of a graph for the most ``important'' vertices.  The sum of the eigenvalues of an adjacency matrix is closely related to the number of triangles in a graph, and insights from the eigenvalues of the Laplacian matrix have given rise to the field of spectral graph theory (see, e.g., \citet{chung1997spectral}).  In dynamical systems, numerical algorithms, and control, eigenstructure helps us analyze local stability and typical behaviors~\citep{luenberger1979introduction,bertsekas2016nonlinear}.

The problem of estimating the eigenvalue distribution of large matrices has a particularly long history in fields such as condensed matter physics and quantum chemistry, where one often wishes to estimate the \emph{density of states} (DOS).  The stochastic methods we develop in this paper build upon these decades of work, which we briefly review here. One of the core components of our procedure is the use of a stochastic estimator for the trace of a matrix by constructing a quadratic form with a random vector.  This procedure is generally attributed to \citet{hutchinson1989stochastic}, but it was coincidentally introduced in \citet{skilling1989eigenvalues}.  This latter work by Skilling sketches out an approach to eigenvalue estimation based on generalized traces and Chebyshev polynomials that strongly influences the present paper.  Although Lanczos and Chebyshev polynomial expansions had been used for many years in physics and quantum chemistry (see, e.g., \citet{haydock1972electronic}, \citet{tal1984accurate}), the combination of polynomial expansion and randomized trace estimation as in \citet{skilling1989eigenvalues} led to the \emph{kernel polynomial method} for estimation the density of states in \citet{silver1994densities} and \citet{silver1996kernel}, in which a smoothing kernel is convolved with the spectral density in order to reduce the Gibbs phenomenon.  See \citet{weisse2006kernel} for a review of the kernel polynomial method and a broader historical perspective on these estimation techniques in physics and quantum chemistry.  Outside of the kernel polynomial method specifically, these methods have also found use in quantum chemistry for stochastic Kohn-Sham density functional theory \citep{baer2013self}.

More broadly, these estimators have been used to perform estimation of the density of states, e.g., the spectral density of Hamiltonian systems.  For example, \citet{drabold1993maximum} used moment-constrained random vectors within the trace estimator to form estimates of the DOS based on the maximum entropy principle.  
\citet{jeffrey1997calculation} took a related approach and used the trace estimator to compute first and second moments with which a Gaussian can be fit.  
\citet{zhu1994orthogonal} and \citet{parker1996matrix} used a truncated Chebyshev approximation to the delta function to estimate the spectral density of a Hamiltonian.

Motivated in part by these physics applications, a broader literature has developed around randomized trace estimators applied to polynomial expansions.  For example, \citet{bai1996some}, \citet{bai1996bounds}, and \citet{golub1994matrices} developed a set of fundamentally similar techniques using a quadrature-based view on the estimation of generalized traces, proving upper and lower bounds on the estimates.  These techniques are closely related to those presented here, although we develop unbiased estimates of the matrix function in addition to the unbiased trace estimator, and address the case where the matrix-vector product is itself noisy.  We also view the problem as one of estimating the entire spectral density, rather than the efficient estimation of a single scalar.  \citet{avron2011randomized} also studied theoretical properties of the trace estimator, and developed convergence bounds, later extended by \citet{roosta2015improved}.

\citet{lin2016approximating} provides an excellent survey of different techniques for estimating the spectral density of a real symmetric matrix using matrix-vector products, including the Chebyshev expansions and kernel smoothing built upon in the present paper. \citet{lin2017randomized} proposes variations of these techniques that have better convergence characteristics when the target matrix is approximately low rank.  \citet{xi2017fast} performs a theoretical and empirical investigation of similar ideas for the generalized eigenvalue problem.  Closely related to the general DOS problem is the estimation of the number of eigenvalues in a particular interval, as such estimates can be used to construct histograms, and \citet{di2016efficient} proposes a class of KPM-like techniques for achieving this.

One function of the eigenspectrum that is of particular interest is the determinant; the log determinant can be viewed as a generalized trace under the matrix logarithm, as described in the proceeding section.  To our knowledge, randomized polynomial-based estimates of the log determinant were first introduced by \citet{skilling1989eigenvalues}, but were later rediscovered by \citet{bai1996some} and \citet{barry1999monte} and then again by \citet{han2015large} and \citet{boutsidis2015randomized}, albeit with more theoretical investigation.  An alternative randomized scheme for the determinant was proposed in \citet{saibaba2016randomized}.  Finally, there have also been recent proposals that incorporate additional structural assumptions in forming such estimates \citep{peng2015large,fitzsimons2017entropic,fitzsimons2017bayesian}.

Similarly, matrix inversion can be viewed as a transformation of the eigenspectrum and can be approximated via polynomials.  The Neumann series is one classic example---essentially the Taylor series of the matrix inverse---when the spectral radius is less than one.  This approximation has been used for \emph{polynomial preconditioning}, originally proposed in \citet{rutishauser1959theory}, although see also \citet{johnson1983polynomial} and \citet{saad1985practical}.  More recently, these series have been proposed as an efficient way to solve linear systems with the Hessian in support of second-order stochastic optimization \citep{agarwal2016second}, using a very similar procedure to that described in Section~\ref{sec:series} but with Neumann series.

Although this history represents a rich literature on randomized estimation of the eigenspectra of implicit matrices, this paper makes several distinct contributions.  First, we introduce a randomized variation of the Chebyshev polynomial expansion that leads to provably \emph{unbiased} estimates when combined with the stochastic trace estimator.  Second, we show that such estimators can be generalized to the case where the matrix-vector product is itself randomized, which is of critical interest to stochastic optimization methods for large scale problems in statistics and machine learning.  Unbiasedness is preserved in this case using generalizations of the Poisson estimator~\citep{wagner1987unbiased,fearnhead2010random,papas2011poisson}.  Third, we introduce a new order-independent smoothing kernel for Chebyshev polynomials, with appealing properties.  Finally, we apply this ensemble of techniques at large scale to matrices with known spectra, in order to validate the methods.

%% file: formulation.tex
\section{Problem Formulation}
\label{sec:problem}
The matrix of interest~$\brmA$ is assumed to be real, square, symmetric, and and of dimension~$D$, i.e.,~${\brmA\in\reals^{D \times D}}$ where ${\rmA_{ij}=\rmA_{ji}}$.  Many of the results in this paper generalize to complex and asymmetric matrices, but the real symmetric case is of primary interest for Hessians and undirected graphs.  The primary challenge to be addressed is that~$D$ can be very large.  For example, the modern stochastic optimization problems arising from the fitting of neural networks may have tens of billions of variables \citep{coates2013deep} and the graph of web pages has on the order of a trillion vertices \citep{webpages}.  When~$\brmA$ is this large, it cannot be explicitly represented as a matrix, but can only be indirectly accessed via matrix-vector products of the form~$\brmA\brmx$, where we can choose~$\brmx\in\reals^D$.  An additional challenge arises in the case of large scale model fitting, where the loss surface may be a sum of billions of terms, and so even~$\brmA\brmx$ cannot be feasibly computed.  Instead, an unbiased but randomized estimator~$\hat{\brmA}\brmx$ is formed from a ``mini-batch'' of data, as in stochastic optimization \citep{robbins1951stochastic}.

We will assume that all of the eigenvalues of~$\brmA$ are within the interval~$(-1,1)$, something achievable by dividing~$\brmA$ by its operator norm plus a small constant.  With this assumption, the abstraction we choose for investigating~$\brmA$ is the \emph{generalized trace}~$\trace(F(\brmA))$, where ${F:\reals^{D \times D}\to\reals^{D \times D}}$ is an operator function that generalizes a scalar function ${f:(-1,1)\to\reals}$ to square matrices.  We take~$f$ to have a Chebyshev expansion~${\gamma_0, \gamma_1, \ldots}$:
\begin{align}
\label{eqn:scalar-cheb}
f(a) &= \sum_{k=0}^\infty\gamma_k T_k(a)\,,
\end{align}
where the functions~${T_k:(-1,1)\to[-1,1]}$ are the Chebyshev polynomials of the first kind.  These are defined via the recursion
\begin{align}
T_0(a) &= 1 & T_1(a) &= a & T_k(a) &= 2aT_{k-1}(a)-T_{k-2}(a)\,.
\end{align}
Chebyshev approximations are appealing as they are optimal under the~$L_\infty$ norm.   Other polynomial expansions are optimal under different function norms, but this generalization is not considered here.  We can generalize the Chebyshev polynomials to square matrices using the recursion
\begin{align}
T_0(\brmA) &= \identity & T_1(\brmA) &= \brmA & T_k(\brmA) &= 2\brmA T_{k-1}(\brmA) - T_{k-2}(\brmA)\,.
\end{align}
\begin{proposition}
Let~$\brmA$ be diagonalizable into~${\brmA=\brmU^{\trans}\bLambda\brmU}$ with orthonormal~$\brmU$ and diagonal~$\bLambda$.  Then the matrix Chebyshev polynomial~$T_k(\brmA)$ applies the scalar Chebyshev polynomial~$T_k(a)$ to each of the eigenvalues of~$\brmA$.
\end{proposition}
\begin{proof}
The proposition is true by inspection for~$T_0(\brmA)$ and~$T_1(\brmA)$.  We proceed inductively using the two-term recursion.  Assume that the proposition is true for~$T_{k-1}(\brmA)$ and~$T_{k-2}(\brmA)$, then
\begin{align}
T_{k-1}(\brmA) &= \brmU^{\trans}T_{k-1}(\bLambda)\brmU &
T_{k-2}(\brmA) &= \brmU^{\trans}T_{k-2}(\bLambda)\brmU\,.
\end{align}
Applying the Chebyshev recursion:
\begin{align}
T_k(\brmA) &= 2\brmA T_{k-1}(\brmA) - T_{k-2}(\brmA)\\
&= 2\brmU^{\trans}\bLambda\brmU\brmU^{\trans}T_{k-1}(\bLambda)\brmU - \brmU^{\trans}T_{k-2}(\bLambda)\brmU\\
&= 2\brmU^{\trans}\bLambda T_{k-1}(\bLambda)\brmU - \brmU^{\trans}T_{k-2}(\bLambda)\brmU\\
&= \brmU^{\trans}\left(2\bLambda T_{k-1}(\bLambda) - T_{k-2}(\bLambda)\right)\brmU\\
&= \brmU^{\trans} T_k(\bLambda) \brmU\,.
\end{align}
\end{proof}
When~$f(a)$ has Chebyshev coefficients~$\gamma_1, \gamma_2,\ldots$, we thus define~$F(\brmA)$ as
\begin{align}
\label{eqn:matrix-cheb}
F(\brmA) &= \sum_{k=0}^\infty \gamma_k T_k(\brmA) \,.
\end{align}
This matrix function can now be seen to be simply applying~$f(a)$ to the eigenvalues of~$\brmA$.  Returning to the generalized trace, we see that it corresponds to first applying $f$ to each eigenvalue and then summing them:
\begin{align}
\trace(F(\brmA)) &= \sum_{d=1}^D f(\lambda_d(\brmA))= \sum_{d=1}^D \sum_{k=0}^\infty \gamma_k T_k(\lambda_d(\brmA))
= \sum_{k=0}^\infty \gamma_k \trace\big( T_k(\brmA) \big)
\,,
\end{align}
where the~$\lambda_d(\brmA)$ are the eigenvalues of~$\brmA$. By computing the traces of the Chebyshev polynomials, we can then can choose different~$f$ to ask different questions about the spectra of~$\brmA$.  In the proceeding sections, we will describe the construction of estimators~$\hat{S}$ such that~${\mathbb{E}[\hat{S}] = \trace(F(\brmA))}$.  

The generalized trace is a surprisingly flexible tool for interrogating the spectrum.  In particular, it enables us to look at the eigenvalues via the normalized spectral measure, which is a sum of Dirac measures:
\begin{align}
\label{eqn:discrete-density}
\psi(\lambda) &= \frac{1}{D}\sum_{d=1}^D \delta(\lambda - \lambda_d(\brmA))\,.
\end{align}
A generalized trace~$\trace(F(\brmA))$ can thus be seen as an expectation under this spectral measure:
\begin{align}
\trace(F(\brmA)) &= D\int^{1}_{-1} f(\lambda)\,\psi(\lambda)\,\diff\lambda
= \sum_{d=1}^D f(\lambda_d(\brmA))\,.
\end{align}
Through the spectral measure, one can compute empirical summaries of the matrix of interest, such as \emph{How many eigenvalues are within~$\epsilon$ of zero?} or \emph{How many negative eigenvalues are there?}.  Such questions can be viewed as particular instances of functions~$f$ where they are, e.g., indicator-like functions.

In this paper, we will take this general view and seek to construct a general purpose spectral density estimate from the Chebyshev traces of~$\brmA$ using a smoothing kernel~$K(\lambda,\lambda')$.  If ${K(\lambda,\lambda')\geq 0}$ and~${\int^{1}_{-1}K(\lambda,\lambda')\,d\lambda =1}$, then we can define the smoothed spectral density
\begin{align}
\tilde{\psi}(\lambda) &= \frac{1}{D}\sum_{d=1}^D K(\lambda, \lambda_d(\brmA))\,.
\end{align}
In the generalized trace framework, we can ask pointwise questions about~$\tilde{\psi}$ by constructing a~$\lambda$-specific~${f_\lambda(a) = K(\lambda, a)}$.  We are seeking to estimate a one-dimensional density, so a dense collection of pointwise estimates is useful for visualization, diagnosis, or to ask various \emph{post hoc} questions such as those above.  As in other kinds of density estimation, smoothing is sensible as it leads to practical numerics and we typically do not care about high sensitivity to small perturbations of the eigenvalues.  Smoothing also minimizes the effects of the Gibbs phenomenon when using Chebyshev polynomials.  Various smoothing kernels have been proposed for the case where the polynomials are truncated (See \cite{weisse2006kernel} for a discussion.), but in this work we will not have a fixed-order truncation and so we will derive a kernel that is 1)~appropriate for density estimation, 2)~reduces the Gibbs phenomenon, 3)~maps naturally to~$(-1,1)$, and 4)~does not depend on the order of the approximation. 

We will break down the task of computing estimates of generalized traces into two parts: the application of~$F$ to the matrix of interest (Section~\ref{sec:series}), and the computation of the trace of the result (Section~\ref{sec:traces}).  In both cases, we will use randomized algorithms to form unbiased estimates.

%% file: power.tex
\section{Randomized Matrix Power Series Estimation}
\label{sec:series}

In this section we discuss the construction of unbiased Monte Carlo estimates
of functions applied to symmetric matrices.  Two forms of Monte Carlo estimator
will be used: 1) multiplications of independent random matrix variates in order
to get unbiased estimates of matrix powers, and 2) Monte Carlo estimation of
infinite polynomial series.  These are framed in terms of randomized Chebyshev
polynomials, slightly generalizing the scalar-valued power series importance
sampling estimator described in~\citet[Section 4.6]{papas2011poisson}.

The core observation in the randomized estimation of infinite series is that sums such as Equations~\eqref{eqn:scalar-cheb} and \eqref{eqn:matrix-cheb} can be turned into expectations by introducing a proposal distribution that has support on the non-negative natural numbers~$\naturals_0$.  The estimator is then a weighted random power or polynomial of random order.  We are here making the further assumption that we only have randomized estimates of the matrix and so these matrix powers must be constructed with care: the power of an expectation is not the expectation of the power.  Instead, we use a product of independent estimates of the matrix, with the $k$th independent estimate denoted~$\hat{\brmA}_k$, to construct an unbiased estimate of the matrix power.  For example, in the case of fitting a large statistical or machine learning regression model, each Hessian estimate~$\hat{\brmA}_k$ might arise from small and independent subsets of data.  This relieves us of the burden of computing a Hessian that is a sum over millions or billions of data.  The following lemma addresses such an unbiased construction for Chebyshev polynomials.
\begin{lemma}
\label{lemma:chebyshev}
Let~${\hat{\brmA}_1, \hat{\brmA}_2, \ldots}$ be a sequence of independent, unbiased estimators of a symmetric real matrix~${\brmA\in\reals^{D \times D}}$.  From this sequence, construct a new sequence as follows:
\begin{align}
\label{eq:chebiter}
\hat{\brmT}_0 &= \identity &
\hat{\brmT}_1 &= \hat{\brmA}_1 &
\hat{\brmT}_k &= 2\hat{\brmA}_{k} \hat{\brmT}_{k-1} - \hat{\brmT}_{k-2}\,.
\end{align}
Then $\hat{\brmT}_k$ is an unbiased estimate of the matrix-valued Chebyshev polynomial of the first kind~$T_k(\brmA)$.
\end{lemma}
\begin{proof}
Unbiasedness for~${k\in\{0,1\}}$ is clear by inspection.  For~${k>1}$, we proceed inductively by assuming that~$\hat{\brmT}_{k-1}$ and~$\hat{\brmT}_{k-2}$ are unbiased.  $\hat{\brmA}_k$ is independent of~$\hat{\brmT}_{k-1}$ and therefore
\begin{align}
\mathbb{E}[2\hat{\brmA}_{k} \hat{\brmT}_{k-1} - \hat{\brmT}_{k-2}]
&= 2\mathbb{E}[\hat{\brmA}_k]\mathbb{E}[\hat{\brmT}_{k-1}] - \mathbb{E}[\hat{\brmT}_{k-2}] = 2\brmA T_{k-1}(\brmA) - T_{k-2}(\brmA)\,.
\end{align}
\end{proof}

Although this randomized estimator is unbiased, its variance explodes as a function of~$k$. The intuitive reason for this is that deterministic Chebyshev polynomials have a somewhat magical property: even though the leading coefficients increase exponentially with~$k$, the terms cancel out in such a way that the sum is still within the interval~$[-1,1]$.  However, in the case here the polynomial is random.  The noise overwhelms the careful interaction between terms and results in an explosion of variance.  However, it is possible to find an upper bound for this variance and identify situations where the Chebyshev coefficients decay fast enough to achieve finite variance.
For a $D\times D$ matrix $\brmA$, we denote its spectral norm by $\|\brmA\|$ and its Frobenius norm by 
$\|\brmA\|_F = (\sum_{i,j=1}^{D}A_{ij}^2)^{1/2}$.
\begin{lemma}
\label{lemma:cheb-variance}
Let~${\hat{\brmA}_1, \hat{\brmA}_2, \ldots}$ be a sequence of independent, unbiased estimators of a 
symmetric real matrix~${\brmA\in\reals^{D \times D}}$ and let $\hat{\brmT}_k$ be defined as in~\eqref{eq:chebiter}.
Suppose that 
\[ 4\mathbb{E}[\|\hat{\brmA}_k\|^2] + 2\mathbb{E}[\|\hat{\brmA}_k\|] + 1\leq \alpha\quad\textup{for all $k$.}\]
Then, for all $k\geq 0$, 
\[\mathbb{E}[\|\hat{\brmT}_k\|_F^2] \leq D\alpha^k.\]
\end{lemma}
\begin{proof}
	Let $t_k = \mathbb{E}[\|\hat{\brmT}_k\|_F^2]$ for $k\geq 0$. For each $k\geq 1$ define
	\[ \brmB_k = \mathbb{E}\begin{bmatrix}\hat{\brmT}_k\hat{\brmT}_{k}^T & \hat{\brmT}_k\hat{\brmT}_{k-1}^T\\
			\hat{\brmT}_{k-1}\hat{\brmT}_k^T & \hat{\brmT}_{k-1}\hat{\brmT}_{k-1}^T\end{bmatrix}\]
	and note that $\trace(\brmB_k) = t_k + t_{k-1}$ for all $k\geq 1$ and
	that $\brmB_k$ is positive semidefinite for all $k\geq 1$. 
	We can write the defining recurrence for $\hat{\brmT}_k$ as
	\[ \begin{bmatrix} \hat{\brmT}_k\\\hat{\brmT}_{k-1}\end{bmatrix} = 
		\begin{bmatrix}2\hat{\brmA}_k & -\brmI\\\brmI & 0\end{bmatrix} 
		\begin{bmatrix}\hat{\brmT}_{k-1}\\\hat{\brmT}_{k-2}\end{bmatrix}\quad\textup{for $k\geq 2$}.\]
	This gives the following recurrence for $\brmB_k$: 
	\begin{align*}
		 \brmB_{k} &  = \mathbb{E}\left[\begin{bmatrix}2\hat{\brmA}_k & -\brmI\\\brmI & 0\end{bmatrix}
				\begin{bmatrix}\hat{\brmT}_{k-1}\hat{\brmT}_{k-1}^T & \hat{\brmT}_{k-1}\hat{\brmT}_{k-2}^T\\
				\hat{\brmT}_{k-2}\hat{\brmT}_{k-1}^T & \hat{\brmT}_{k-2}\hat{\brmT}_{k-2}^T\end{bmatrix}
				\begin{bmatrix}2\hat{\brmA}_k & -\brmI\\\brmI & 0\end{bmatrix}^T\right]\\
			& = \mathbb{E}_{\hat{\brmA}_{k}}\left[
					\begin{bmatrix}2\hat{\brmA}_k & -\brmI\\\brmI & 0\end{bmatrix}\brmB_{k-1}
					\begin{bmatrix}2\hat{\brmA}_k & -\brmI\\\brmI & 0\end{bmatrix}^T\right]
	\end{align*}
	since $\hat{\brmA}_k$ is independent of $\hat{\brmA}_{k-1},\ldots, \hat{\brmA}_1$. 
	Taking the trace of both sides gives
	\begin{align*}
		 t_{k} + t_{k-1} & = \trace\left(\mathbb{E}_{\hat{\brmA}_{k}}
		\left[\begin{bmatrix}2\hat{\brmA}_k & -\brmI\\\brmI & 0\end{bmatrix}\brmB_{k-1}
		\begin{bmatrix}2\hat{\brmA}_k & -\brmI\\\brmI & 0\end{bmatrix}^T\right]\right)\\
		& = \mathbb{E}_{\hat{\brmA}_k}\left[\trace\left(
			\begin{bmatrix}2\hat{\brmA}_k & -\brmI\\\brmI & 0\end{bmatrix}^T
			\begin{bmatrix}2\hat{\brmA}_k & -\brmI\\\brmI & 0\end{bmatrix}\brmB_{k-1}\right)\right]\\
		& \leq \mathbb{E}_{\hat{\brmA}_k}\left[
			\left\|\begin{bmatrix} 4\hat{\brmA}_k^T\hat{\brmA}_k+\brmI & -2\hat{\brmA}_k^T\\
							-2\hat{\brmA}_k & \brmI\end{bmatrix}\right\|\right](t_{k-1}+t_{k-2})\\
		& \leq \mathbb{E}_{\hat{\brmA}_k}\left[
		\left\|\begin{bmatrix} 4\hat{\brmA}_k^T\hat{\brmA}_k + \brmI & \mathbf{0}\\\mathbf{0} & \brmI\end{bmatrix}\right\|
 		+ \left\|\begin{bmatrix} \mathbf{0} & -2\hat{\brmA}_k^T\\-2\hat{\brmA}_k & \mathbf{0}\end{bmatrix}\right\|
			\right](t_{k-1}+t_{k-2})\\
		& \leq (4\mathbb{E}[\|\hat{\brmA}_k\|^2] + 2\mathbb{E}[\|\hat{\brmA}_k\|] + 1)(t_{k-1} + t_{k-2})\\
		& \leq \alpha(t_{k-1} + t_{k-2})
	\end{align*}
	where we have use the inequality $\trace(\brmA\brmB) \leq \|\brmA\|\trace(\brmB)$ 
	for positive semidefinite matrices $\brmA$ and $\brmB$, together with the triangle inequality, and the definition of $\alpha$. 
	We then see that for all $k\geq 2$, 
	\begin{equation}
	\label{eq:tbound} t_{k} \leq t_{k} + t_{k-1} \leq \alpha^{k-1}(t_1 + t_0).
	\end{equation}
	Since $t_0 = \mathbb{E}[\|\brmI\|_F^2] = D$ and 
	\[ t_1 = \mathbb{E}[\|\hat{\brmA}_1\|_F^2] \leq D\mathbb{E}[\|\hat{\brmA}_1\|^2] 
			\leq D(4\mathbb{E}[\|\hat{\brmA}_1\|^2] + 2\mathbb{E}[\|\hat{\brmA}_1\|])\leq D(\alpha - 1)\]
	it follows that $ t_k \leq D\alpha^k$
	for all $k\geq 0$.
\end{proof}

Fortunately, in the situation of interest here, the coefficients weighting each of the Chebyshev terms fall off even more rapidly than the variance increases.  In particular, the von Mises smoothing introduced in Section~\ref{sec:vonmises} has subexponential tails that counteract the explosion of the noisy Chebyshev polynomials.  Thus with independent estimators to form stable unbiased estimates of the weighted Chebyshev polynomials, we can now form an unbiased estimate of~$F(\brmA)$.
\begin{proposition}
\label{prop:fa}
Let~${F:\reals^{D\times D}\to\reals^{D \times D}}$ be an analytic operator function with Chebyshev polynomial coefficients~${\gamma_0, \gamma_1, \ldots}$.  Construct the sequence~${\hat{\brmT}_0, \hat{\brmT}_1, \ldots}$ as in Lemma~\ref{lemma:chebyshev} above.  Let~$\bpi$ be a probability distribution on the natural numbers (including zero)~$\naturals_0$, where~$\pi_k$ is the probability mass function and~${\pi_k > 0}$ if~${\gamma_k \neq 0}$.  Construct an estimator~$\hat{\brmF}$ by first drawing a random nonnegative integer~$k$ from~$\bpi$ and then computing~${\hat{\brmF}=\gamma_k\hat{\brmT}_k/\pi_k}$.  Then ${\mathbb{E}[\hat{\brmF}]=F(\brmA)}$.
\end{proposition}
\begin{proof}
This simply combines importance sampling with the Chebyshev estimator from Lemma~\ref{lemma:chebyshev}:
\begin{align}
\mathbb{E}[\hat{\brmF}] = \mathbb{E}\left[\frac{\gamma_k\hat{\brmT}_k}{\pi_k}\right] &=
\sum_{k=0}^\infty \pi_k \frac{\gamma_k\mathbb{E}[\hat{\brmT}_k]}{\pi_k}
= \sum_{k=0}^\infty \gamma_k T_{k}(\brmA) = F(\brmA)\,.
\end{align}
\end{proof}

The variance of importance sampling estimates can be poor.  In the following, we derive the proposal distribution that minimizes the variance of the estimator under the Frobenius norm, i.e., it minimizes the expected squared Frobenius norm of the difference between the estimator~$\hat{\brmF}$ and its mean~$F(\brmA)$.
\begin{proposition}
	Suppose the sequence $\gamma_0,\gamma_1,\ldots$ is such that  
\[ Z = \sum_{k'=0}^\infty |\gamma_{k'}|\sqrt{\mathbb{E}_{\hat{\brmT}_{k'}}[\|\hat{\brmT}_{k'}\|^2_F]}\]
is finite. Then the proposal distribution~$\bpi^\star$ where
\begin{align}
  \pi^\star_k &= \frac{1}{Z} |\gamma_k|\sqrt{\mathbb{E}_{\hat{\brmT}_k}[\|\hat{\brmT}_k\|^2_F]},
  \qquad 
\end{align}
minimizes~$\mathbb{E}_{\bpi,\hat{\brmT}}[\|\hat{\brmF}-F(\brmA)\|^2_F\|]$.
\end{proposition}
\begin{proof}
The Frobenius norm can be written as a trace, so that
\begin{align}
\mathbb{E}_{\bpi,\hat{\brmT}}[\|\hat{\brmF} - F(\brmA)\|^2_F] &= \mathbb{E}_{\bpi,\hat{\brmT}}[
\trace((\hat{\brmF}-F(\brmA))^{\trans}(\hat{\brmF}-F(\brmA)))]\\
&= \mathbb{E}_{\bpi,\hat{\brmT}}[\trace(\hat{\brmF}^{\trans}\hat{\brmF})] - \trace((F(\brmA))^{\trans}F(\brmA))\\
&= \mathbb{E}_{\bpi,\hat{\brmT}}[\|\hat{\brmF}\|^2_F] - \|F(\brmA)\|^2_F\,.
\end{align}
The second term is independent of $\bpi$, so it suffices to minimize the first
term. Under $\bpi^\star$ this term can be expanded as
\begin{align}
\mathbb{E}_{\bpi^\star,\hat{\brmT}}[\|\hat{\brmF}\|^2_F] &= \mathbb{E}_{\bpi^\star}\left[
\mathbb{E}_{\hat{\brmT}_k}\left[
\frac{\gamma_k^2}{(\pi^\star_k)^2}\|\hat{\brmT}_k\|^2_F
\right]
\right] = \mathbb{E}_{\bpi^\star}\left[
\frac{\gamma_k^2}{(\pi^\star_k)^2}
\mathbb{E}_{\hat{\brmT}_k}\left[
\|\hat{\brmT}_k\|^2_F
\right]
\right]\\
&= \sum_{k=0}^\infty \frac{\gamma_k^2}{\pi^\star_k}\mathbb{E}_{\hat{\brmT}_k}\left[
\|\hat{\brmT}_k\|^2_F
\right]
= Z \sum_{k=0}^\infty |\gamma_k|\sqrt{\mathbb{E}_{\hat{\brmT}_k}[\|\hat{\brmT}_k\|^2_F]} = Z^2 \,.
\end{align}
To show that this value is minimal, we consider an alternative
distribution~$\bpi$ and apply Jensen's inequality on the
convex function $z \mapsto z^2$ to write
\begin{equation}
Z^2 =  \left(\sum_{k=0}^\infty \pi_k\frac{|\gamma_k|\sqrt{\mathbb{E}_{\hat{\brmT}_k}[\|\hat{\brmT}_k\|^2_F]}}{\pi_k} \right)^2
\leq \sum_{k=0}^\infty \pi_k \frac{\gamma_k^2}{\pi_k^2} \mathbb{E}_{\hat{\brmT}_k}[\|\hat{\brmT}_k\|^2_F]
= \mathbb{E}_{\bpi,\hat{\brmT}}\left[
\|\hat{\brmF}\|^2_F
\right]\,.
\end{equation}
Thus we can see that~$\bpi^\star$ offers expected squared Frobenius norm less than or equal to that expected under all other possible proposals~$\bpi$.
\end{proof}
We also note that an importance sampler can be constructed with a hazard rate formulation, resulting in a ``randomized truncation'' that is still unbiased, using an estimator with~$k\sim\bpi$ as before:
\begin{align}
\hat{\brmF}_{\sf{truncation}} &= \gamma_0\hat{\brmT}_0 + \sum_{j=1}^k \frac{\gamma_j \hat{\brmT}_j}{1 - \sum_{\ell=0}^{j-1}\pi_{\ell}}\,
.\end{align}
Although this is appealing because it makes use of the intermediate steps of the recursive computation of~$\hat{\brmT}_k$, we have not been able to identify hazard rates that generally result in variance reduction.  It nevertheless seems likely that such hazard rates exist.

\subsection{Von Mises Smoothing}
\label{sec:vonmises}
As discussed in Section~\ref{sec:problem}, it is useful to frame the problem in terms of estimating the spectral measure~$\psi(\lambda)$.    This is closely related to the \emph{kernel polynomial method} (see, e.g., \citet{weisse2006kernel}) in which truncated Chebyshev approximations are smoothed with order-dependent Fej\'{e}r, Jackson, or Lorentz kernels to minimize the Gibbs phenomenon in the approximation.  Here we are not making fixed-order truncations, but smoothing is still desirable.  We introduce the \emph{von Mises kernel} to address our situation.  It has support on the interval~$(-1,1)$, integrates to one, and is approximately Gaussian for large~$\kappa$, but with a convenient closed-from Chebyshev expansion: 
\begin{align}
K_{\kappa}(\lambda,\lambda') &= \frac{
e^{\kappa \cos(\cos^{-1}(\lambda) - \cos^{-1}(\lambda'))}
+ e^{\kappa \cos(\cos^{-1}(-\lambda) - \cos^{-1}(\lambda') + \pi)}
}{2\pi I_0(\kappa) \sqrt{1-\lambda^2}}\,,
\end{align}
where~$I_\eta(\kappa)$ is the modified Bessel function of the first kind.  This kernel arises from constructing a von Mises distribution on the unit circle and then projecting it down onto the interval~$(-1,1)$, along with the appropriate Jacobian adjustment.  The parameter~$\kappa$ is inversely related to the squared width of the kernel.  Figure~\ref{fig:von_mises} shows several densities from different parameters, on the unit circle and projected.

\begin{figure}[t]
\centering
\begin{subfigure}[b]{0.45\textwidth}
  \includegraphics[width=\textwidth]{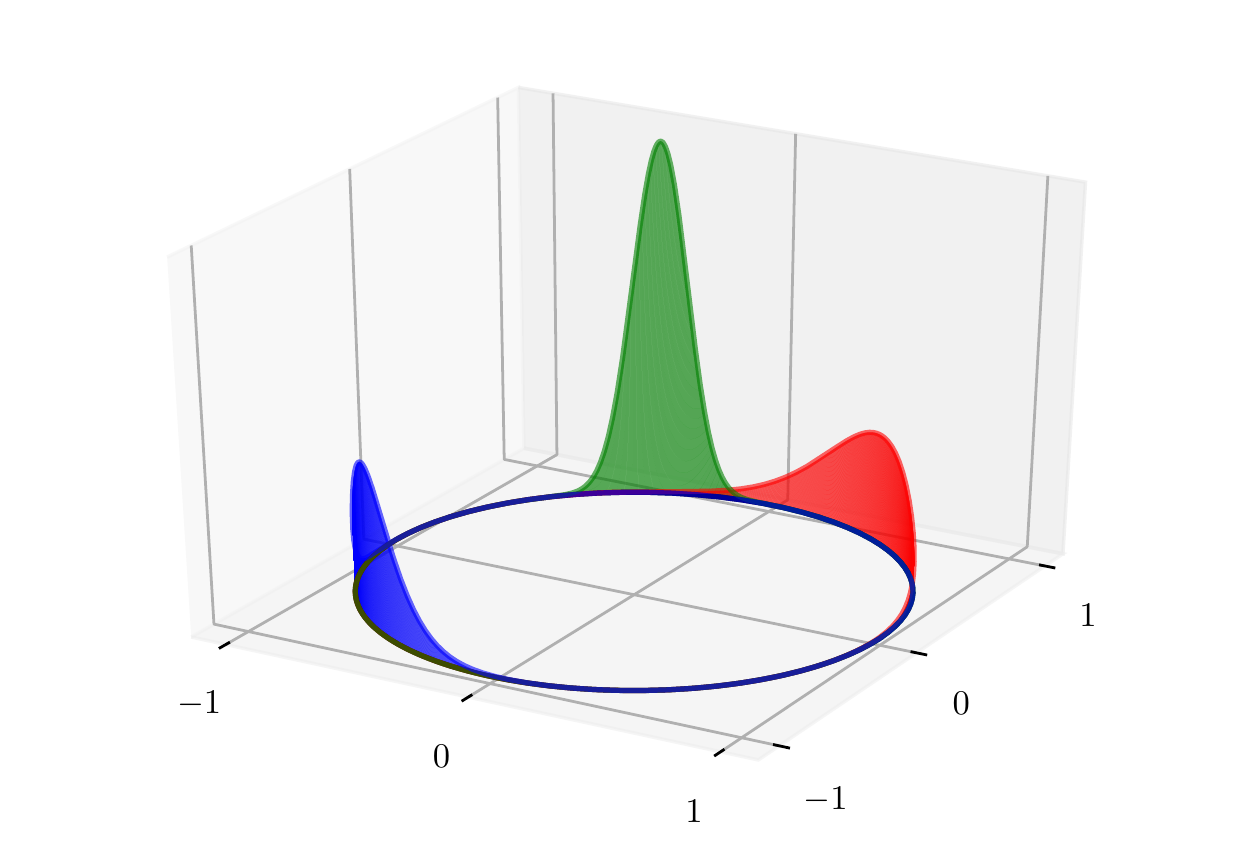}
  \label{fig:von_mises_3d}
  \caption{Von Mises distributions on the unit circle}
\end{subfigure}
\begin{subfigure}[b]{0.45\textwidth}
\includegraphics[width=\textwidth]{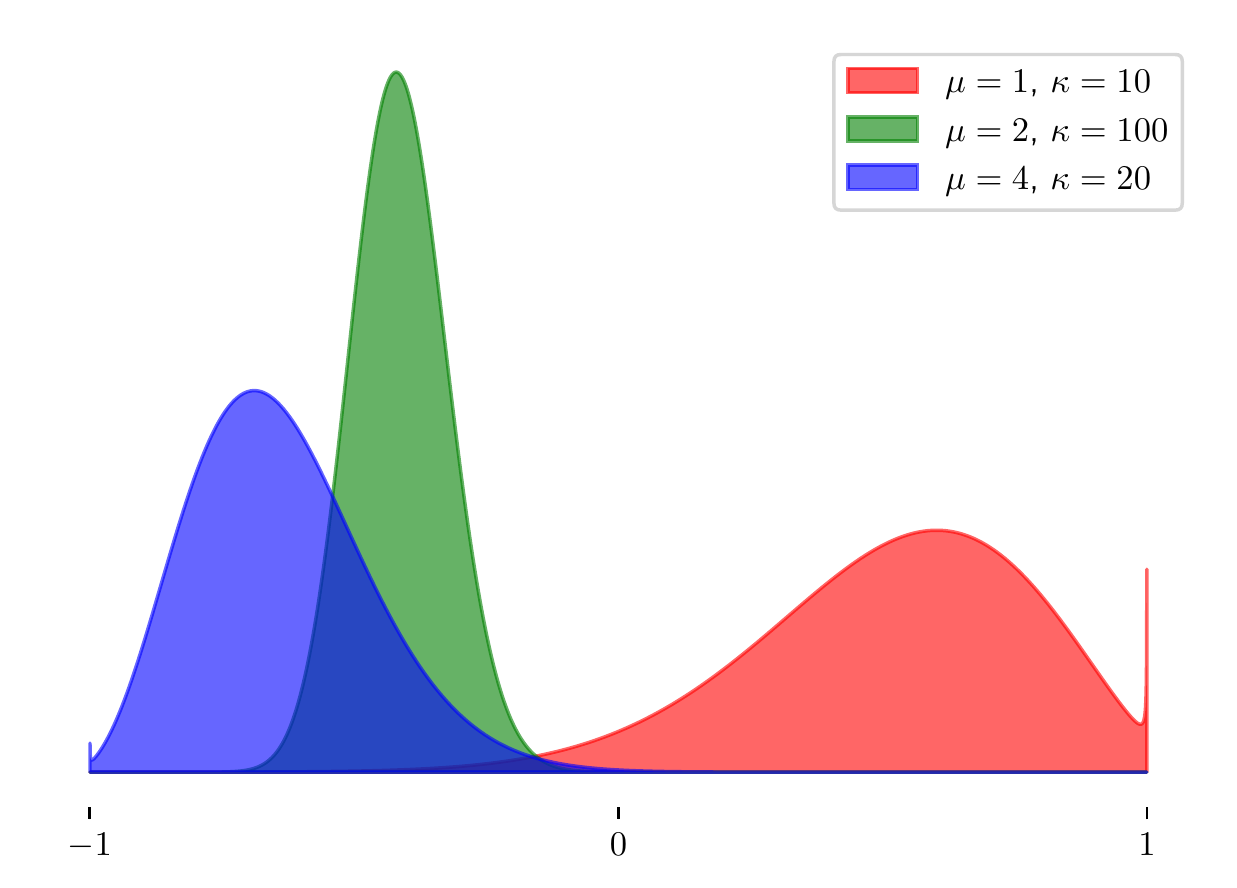}
\label{fig:von_mises_2d}
\caption{Von Mises projected onto $(-1,1)$}
\end{subfigure}
\caption{Three von Mises densities are shown with different parameters.  On the left they are shown ``natively'' on the unit circle.  On the right they have been projected down to~$(-1,1)$.  Note that the measure correction can result in boundary effects such as the large density values near~$-1$ and~$1$.}
\label{fig:von_mises}
\end{figure}

Of particular interest is that the von Mises kernel has a closed-form Chebyshev polynomial expansion for any~$\lambda'$.  Chebyshev polynomials can be constructed in three different ways.  First, on the interval~$(-1,1)$ they can be constructed via the recursion in Section~\ref{sec:problem} or as~${T_k(x) = \cos(k\cos^{-1}(x))}$.  They can also be computed on the unit circle in the complex plane as~${T_k(x) = \frac{1}{2}(z^k + z^{-k})}$ where~${z = \cos^{-1}(x) + i\sin^{-1}(x)}$.  Finally, they can be constructed via the phase alone, as~${|z|=1}$ on the unit circle, i.e.,~${z=e^{i\pi\theta}}$, so that~${T_k(x) = \cos(k \theta)}$.  In the~$z$ and~$\theta$ cases, the polynomials are projected onto the interval~$(-1,1)$.  See \citet{trefethen2013approximation} for a comprehensive treatment of these representations.

We focus only on the phase view.  The von Mises distribution is a probability distribution defined on the unit circle, with density given by
\begin{align}
f(\theta\,|\,\mu,\kappa) &= \frac{e^{\kappa \cos(\theta-\mu)}}{2 \pi I_0(\kappa) }\,.
\end{align}
Its variance is given by~${1-I_1(\kappa)/I_0(\kappa)}$.  For a fixed~$\mu$ and~$\kappa$, the coefficients can be computed via
\begin{align}
\gamma_k &= \frac{2}{\pi} \int^{\pi}_{-\pi} f(\theta\,|\,\mu,\kappa)\,\cos(k\theta)\,d\theta
= \frac{1}{\pi^2 I_0(\kappa)}\int^{\pi}_{-\pi} \exp\{\kappa \cos(\theta-\mu)\}\,\cos(k\theta)\,d\theta\\
&= \frac{2}{\pi}\frac{I_k(\kappa)}{I_0(\kappa)}\cos(k\mu)\,.
\end{align}
The first coefficient for Chebyshev polynomials does not have the factor of two and so~${\gamma_0=1/\pi}$.  Returning to the interval~$(-1,1)$ we have
\begin{align}
\label{eqn:vm-coef0}
\gamma_0(\lambda',\kappa) &= \frac{1}{\pi}\,,\\
\label{eqn:vm-coefk}
\gamma_k(\lambda',\kappa) &= \frac{2}{\pi}\frac{I_k(\kappa)}{I_0(\kappa)}\cos(k\cos^{-1}(\lambda'))\,.
\end{align}
For our purposes here,~${\kappa\approx 5000}$ and so the ratio of modified Bessel functions is well approximated as:
\begin{align}
\frac{I_k(\kappa)}{I_0(\kappa)} &\approx \exp\left\{-\frac{1}{2}\frac{k^2}{\kappa}\right\}\,.
\end{align}
This approximation is computationally useful and also makes it clear that the~$\gamma_k$ rapidly converge to zero as a function of~$k$.

\subsection{Bounding the Variance of the Importance Samples}
\label{sec:bounding}
When computing an importance sampled estimate of the von Mises kernel, a natural choice of proposal distribution is to have it be proportional to the absolute value of the coefficients in Equations~\eqref{eqn:vm-coef0} and~\eqref{eqn:vm-coefk}.  Here the proposal distribution will be denoted as~$\brmq$ rather than~$\bpi$ to distinguish it from the constant~$\pi$:
\begin{align}
\label{eqn:vm-prop0}
\rmq_0(\lambda',\kappa) &= \frac{1}{\mathcal{Z(\lambda',\kappa)}}\,, \\
\label{eqn:vm-propk}
\rmq_k(\lambda',\kappa) &= \frac{2}{\mathcal{Z(\lambda',\kappa)}}\frac{I_k(\kappa)}{I_0(\kappa)}|\cos(k \cos^{-1}(\lambda'))|\,,
\end{align}
where the normalizing constant is given by
\begin{align}
\mathcal{Z(\lambda',\kappa)} &=
1 + \frac{2}{I_0(\kappa)}\sum_{k=1}^\infty I_k(\kappa)|\cos(k \cos^{-1}(\lambda'))|
\,.
\end{align}
Note that the~$1/\pi$ factors are absorbed into~$\mathcal{Z}(\lambda',\kappa)$.  Observing that
\begin{align}
\sum_{k=1}^\infty I_k(\kappa)|\cos(k \cos^{-1}(\lambda'))| &\leq \sum_{k=1}^\infty I_k(\kappa) = \frac{1}{2}(e^\kappa - I_0(\kappa))\,,
\end{align}
we can see that the normalization constant has upper bound
\begin{align}
\mathcal{Z}(\lambda',\kappa) &\leq 1 + \frac{e^\kappa - I_0(\kappa)}{I_0(\kappa)} = \frac{e^\kappa}{I_0(\kappa)} \approx \sqrt{2\pi\kappa}\,.
\end{align}
In the theorem below we use this bound on~$\mathcal{Z}(\lambda',\kappa)$ and the bound on the Chebyshev estimators from Lemma~\ref{lemma:cheb-variance} to provide an upper bound on the expected squared Frobenius norm of the overall importance sampler, as a function of~$\lambda'$ and~$\kappa$. 
\begin{theorem}
The expected squared Frobenius norm of the randomized Chebyshev estimator has upper bound
\begin{align*}
\mathbb{E}\left[ \left\lVert\frac{\gamma_k(\lambda',\kappa)}{\rmq_k(\lambda',\kappa)} \hat{\brmT}_k \right\rVert_F^2 \right] &\leq \frac{D e^\kappa}{\pi^2 I_0(\kappa)}
\left(1 + \frac{2e^\kappa}{I_0(\kappa)}\left(\exp\left\{\alpha\kappa/2\right\}-1\right)\right)\,
\end{align*}
where $4\mathbb{E}[||\hat{\brmA}_k||^2] + 2 \mathbb{E}[||\hat{\brmA}_k||] + 1 \leq \alpha$ for all $k$.
\end{theorem}
\begin{proof}
We first use the law of total expectation and incorporate the bound from Lemma~\ref{lemma:cheb-variance}:
\begin{align}
\mathbb{E}\left[ \left\lVert\frac{\gamma_k(\lambda',\kappa)}{\rmq_k(\lambda',\kappa)} \hat{\brmT}_k \right\rVert_F^2 \right] &= \mathbb{E}_k\left[ 
\left(\frac{\gamma_k(\lambda',\kappa)}{q_k(\lambda',\kappa)}\right)^2
\mathbb{E}_{\hat{\brmT}_k}\left[
\left\lVert \hat{\brmT}_k \right\rVert_F^2
\given k
\right]\right]\\
&\leq D\,\mathbb{E}_k\left[
\left(\frac{\gamma_k(\lambda',\kappa)}{\rmq_k(\lambda',\kappa)}\right)^2\alpha^k
\right]\,.
\end{align}
The~$\gamma_k$ and proposal are chosen to resolve to a simple form:
\begin{align}
\frac{\gamma_k^2(\lambda',\kappa)}{\rmq_k(\lambda',\kappa)} &= \begin{cases}
\displaystyle\frac{\mathcal{Z}(\lambda',\kappa)}{\pi^2} & \text{if $k=0$}\\
\displaystyle 2 \frac{\mathcal{Z}(\lambda',\kappa)}{\pi^2}\frac{
 I_k(\kappa)
}{
I_0(\kappa)
}|\cos(k \cos^{-1}(\lambda'))| & \text{if $k>0$}
\end{cases}\,.
\end{align}
In particular, we get a straightforward bound on this term for~${k>0}$ by discarding the cosine:
\begin{align}
\frac{\gamma_k^2(\lambda',\kappa)}{\rmq_k(\lambda',\kappa)} &\leq 2\frac{\mathcal{Z}(\lambda',\kappa)}{\pi^2}\frac{I_k(\kappa)}{I_0(\kappa)} \qquad\text{when $k>0$}\,.
\end{align}
Substituting this back into the overall bound and using the standard result~${I_k(\kappa) \leq \frac{\kappa^k}{2^k k!}e^\kappa}$,
\begin{align*}
\mathbb{E}\left[ \left\lVert\frac{\gamma_k(\lambda',\kappa)}{\rmq_k(\lambda',\kappa)} \hat{\brmT}_k \right\rVert_F^2 \right] &\leq D\frac{\mathcal{Z}(\lambda',\kappa)}{\pi^2}
\left(
1 + \frac{2}{I_0(\kappa)}
\sum_{k=1}^\infty I_k(\kappa)\alpha^k%\left( 4\mathbb{E}[||\hat{\brmA}||^2] + 2\mathbb{E}[||\hat{\brmA}||]+1\right)^k
\right)\\
&\leq D\frac{\mathcal{Z}(\lambda',\kappa)}{\pi^2}
\left(
1 + \frac{2e^\kappa}{I_0(\kappa)}
\sum_{k=1}^\infty \frac{\kappa^k}{2^k k!}\alpha^k%\left( 4\mathbb{E}[||\hat{\brmA}||^2] + 2\mathbb{E}[||\hat{\brmA}||]+1\right)^k
\right)\\
&= D\frac{\mathcal{Z}(\lambda',\kappa)}{\pi^2}
\left(
1 + \frac{2e^\kappa}{I_0(\kappa)}
\left(
\exp\left\{\alpha \kappa/2\right\}
-1
\right)
\right)\,.
\end{align*}
We can now insert the bound for~$\mathcal{Z}(\lambda',\kappa)$ to get the overall result:
\begin{align*}
\mathbb{E}\left[ \left\lVert\frac{\gamma_k(\lambda',\kappa)}{\rmq_k(\lambda',\kappa)} \hat{\brmT}_k \right\rVert_F^2 \right] &\leq
\frac{D e^\kappa}{\pi^2 I_0(\kappa)}
\left(
1 + \frac{2e^\kappa}{I_0(\kappa)}
\left(
\exp\left\{
\alpha \kappa /2
\right\}
-1
\right)
\right)\,.
\end{align*}
\end{proof}
This bound does not depend on~$k$, but we can see that the variance effectively grows exponentially with~$\kappa$.  This can be seen directly as a bias/variance tradeoff: larger~$\kappa$ corresponds to less smoothing and lower bias with a large variance penalty.

%% file: traces.tex
\section{Skilling-Hutchinson Randomized Estimation of Generalized Traces}
\label{sec:traces}
Having constructed an unbiased estimator of analytic operator functions, we now focus on estimation of the trace of such matrices.  The classic randomized estimator described in \citet{hutchinson1989stochastic} and \citet{skilling1989eigenvalues}, is the following, which we will call the \emph{Skilling-Hutchinson trace estimator}:
\begin{proposition}
\label{prop:skilling-hutchinson}
Let~${\brmA \in \reals^{D \times D}}$ be a square matrix and~${\brmx\in\reals^D}$ be a random vector such that~${\mathbb{E}[\brmx\brmx^{\trans}] = \identity}$.  Then~${\mathbb{E}[\brmx^{\trans}\brmA\brmx]=\trace(\brmA)}$.
\end{proposition}
\begin{proof}
\begin{align*}
  \mathbb{E}[\brmx^{\trans}\brmA\brmx] &= \mathbb{E}[\trace(\brmx^{\trans}\brmA\brmx)] & \text{trace of a scalar}\\
  &= \mathbb{E}[\trace(\brmA\brmx\brmx^{\trans})] & \text{invariance to cyclic permutation}\\
  &= \trace(\mathbb{E}[\brmA\brmx\brmx^{\trans}]) & \text{linearity of trace}\\
  &= \trace(\brmA\mathbb{E}[\brmx\brmx^{\trans}]) & \text{linearity of expectation}\\
  &= \trace(\brmA)
\end{align*}
\end{proof}
The requirement for the random variable~$\brmx$ is fairly weak and is satisfied by, for example, standard normal and Rademacher variates.  The properties of such estimators have been studied in \citet{avron2011randomized}.  Finally, the following theorem makes it clear that this trace estimator can be combined with the randomized Chebyshev expansion to construct an overall unbiased estimate of the generalized trace.
\begin{theorem}
Let~$\hat{\brmF}$ be an unbiased estimate of~$F(\brmA)$ as in Proposition~\ref{prop:fa}.  Then the Skilling-Hutchinson trace estimator by~$\brmx^{\trans}\hat{\brmF}\brmx$ with random~${\brmx\in\reals^D}$ such that~${\mathbb{E}[\brmx\brmx^{\trans}]=\identity}$ (as in Proposition~\ref{prop:skilling-hutchinson}) and~$\brmx$ is independent of~$\hat{\brmF}$, then~${\mathbb{E}[\brmx^{\trans}\hat{\brmF}\brmx]=\trace(F(\brmA))}$.
\end{theorem}
\begin{proof}
Following the simple argument from the proof of Proposition~\ref{prop:skilling-hutchinson}, we have
\begin{align}
\mathbb{E}[\brmx^{\trans}\hat{\brmF}\brmx]=\trace(\mathbb{E}[\hat{\brmF}\brmx\brmx^{\trans}])\,.
\end{align}
Since~$\hat{\brmF}$ is independent of~$\brmx\brmx^{\trans}$, then ${\mathbb{E}[\hat{\brmF}\brmx\brmx^{\trans}]=\mathbb{E}[\hat{\brmF}]\mathbb{E}[\brmx\brmx^{\trans}]=F(\brmA)}$.
\end{proof}
The resulting basic algorithm that combines randomized Chebyshev polynomials, von Mises smoothing, and the Skilling-Hutchinson trace estimator is shown in Algorithm~\ref{alg:basic}. For simplicity, the pseudocode shows a single sample for both the polynomial and the quadratic form.

\begin{algorithm}[t]
\caption{Spectral Density Estimation (Without Control Variates)}
\begin{algorithmic}[1]
\Ensure Smoothing parameter~$\kappa$, query location~$\lambda$, sequence~$\hat{\brmA}_1, \hat{\brmA}_2, \ldots$
\Require Unbiased estimate of the $\kappa$-smoothed spectral density $\tilde{\psi(\lambda)}$
\State $\bpi \gets \left[ |\gamma_0(\lambda,\kappa)|, |\gamma_1(\lambda, \kappa)|, |\gamma_2(\lambda,\kappa)|, \ldots \right] / \sum_{k=0}^\infty |\gamma_k(\lambda, \kappa)|$
\State $k \sim \bpi$
\State $\brmx \sim \distNorm(0, \identity)$
\State $\hat{\brmf}_{-2} \gets \brmx$
\State $\hat{\brmf}_{-1} \gets \hat{\brmA}_1\brmx$ 
\For{$k'\gets 2$ to $k$}
\State $\hat{\brmf} \gets 2\hat{\brmA}_{k'}\hat{\brmf}_{-1} - \hat{\brmf}_{-2}$
\State $\hat{\brmf}_{-2} \gets \hat{\brmf}_{-1}$
\State $\hat{\brmf}_{-1} \gets \hat{\brmf}$
\EndFor
\State \Return $\frac{\gamma_k(\lambda,\kappa)}{\pi_k}\brmx^{\trans}\hat{\brmf}$
\end{algorithmic}
\label{alg:basic}
\end{algorithm}

\subsection{Variance Reduction Through a Control Variate}
The Skilling-Hutchinson trace estimator's variance can be reduced via the construction of a control variate.  Let~${\brmx\sim\distNorm(0,\identity)}$ and form the Skilling-Hutchinson trace estimator as~${\hat{Z}=\brmx^{\trans}\brmA\brmx}$.  A natural choice of control variate is to compute the quadratic form using a second matrix~$\brmB$ which has a known trace.  This is similar to a preconditioner in which one introduces a second related matrix that is somehow easier to deal with than the original.  In this case, the new estimator is
\begin{align}
\hat{Z}' &= \brmx^{\trans}\brmA\brmx - c\cdot (\brmx^{\trans}\brmB\brmx - \trace(\brmB))\,,
  \label{eq:control-variate}
\end{align}
for an appropriately chosen~$c$.  Note that the expectation is preserved, i.e., ${\mathbb{E}[\hat{Z}']=\mathbb[\hat{Z}]=\trace(\brmA)}$.  To compute the variance-minimizing~$c$, we require the following result from \citet{avron2011randomized}:
\begin{proposition}
  The variance of the Skilling-Hutchinson trace estimator applied to a real
  symmetric~$\brmA$ is
  \begin{equation}
    \variance[ \brmx^\trans \brmA \brmx ] = \sum_{i,j,k,l} \rmA_{ij} \rmA_{kl} (\expect[x_i x_j x_k x_l] - \expect[x_i x_j] \expect[x_k x_l])\,,
  \end{equation}
  and in particular when $\brmx \sim \distNorm(0, I)$ we have
  \begin{equation}
    \variance[ \brmx^\trans \brmA \brmx ] = 2\,\trace(\brmA^2) = 2 \| \brmA \|^2_\text{F}\,.
  \end{equation}
\end{proposition}
\begin{proof}
  The general expression for the variance follows from
  \begin{align}
    \variance[ \brmx^\trans \brmA \brmx ]
    &= \sum_{i,j,k,l} \expect[ (x_i \rmA_{ij} x_j)(x_k \rmA_{kl} x_l) ]
       - \expect[ x_i \rmA_{ij} x_j ] \expect[x_k \rmA_{kl} x_l ]
    \\
    &= \sum_{i,j,k,l} \rmA_{ij} \rmA_{kl} (\expect[ x_i x_j x_k x_l ] - \expect[x_i x_j] \expect[x_k x_l])\,.
  \end{align}
  When $\brmx \sim \distNorm(0, I)$, we can use the formulas for the second and fourth moments,
  \begin{equation}
    \expect[ x_i x_j ]
    = \delta_{ij}\,,
    \qquad
    \expect[x_i x_j x_k x_l]
    = \delta_{ij} \delta_{kl} + \delta_{ik} \delta_{jl} + \delta_{il} \delta_{jk}\,,
    \qquad
    \delta_{ij}
    = \begin{cases}
      1 & i = j \\
      0 & i \neq j \\
    \end{cases}
    \,,
  \end{equation}
  resulting in the simplification
  \begin{align}
    \variance[ \brmx^\trans \brmA \brmx ]
    &= \sum_{i,j,k,l} \rmA_{ij} \rmA_{kl} ( \delta_{ik} \delta_{jl} + \delta_{il} \delta_{jk} )
    = \bigg( \sum_{j,k,l} \rmA_{kj} \rmA_{kl} \delta_{jl} \bigg) + \bigg( \sum_{i,j,l} \rmA_{ij} \rmA_{jl} \delta_{il} \bigg)\\
    &= 2 \, \trace(\brmA^2)\,.
  \end{align}

  An alternative proof of the identity for the Gaussian case,
  from~\citet{avron2011randomized}, uses the rotational symmetry of the
  Gaussian.
  Decompose the matrix into~${\brmA=\brmU^{\trans}\bLambda\brmU}$, where~$\brmU$ is orthogonal and~$\bLambda$ is diagonal.  As~$\brmx$ is a spherical Gaussian,~${\brmy=\brmU\brmx}$ is also marginally spherical Gaussian.  Thus
  \begin{align}
  \mathbb{V}[\brmx^{\trans}\brmA\brmx] &= \mathbb{V}[\brmy^{\trans}\bLambda\brmy]
  = \mathbb{V}\left[\sum_{d=1}^D y_d^2 \lambda_d(\brmA)\right]
  = \sum_{d=1}^D\mathbb{V}[ y_d^2 \lambda_d(\brmA) ]
  = \sum_{d=1}^D\mathbb{V}[ y_d^2 ] \lambda^2_d(\brmA)\,.
  \end{align}
  The random variates~$y_d^2$ are independently~$\chi^2$-distributed with one degree of freedom and so have variance~2.  The variance of the overall estimator is then given by twice the sum of the squared eigenvalues, which is equal to twice the squared Frobenius norm of~$\brmA$:
  \begin{align}
  \mathbb{V}[\brmx^{\trans}\brmA\brmx] &= 2\sum_{d=1}^D \lambda^2_d(\brmA)
  = 2\,\trace(\brmA^2) = 2\,\|\brmA\|^2_F\,.
  \end{align}
\end{proof}

\begin{algorithm}[t]
\caption{Spectral Density Estimation (With Control Variate)}
\begin{algorithmic}[1]
\Ensure Smoothing param~$\kappa$, query location~$\lambda$, sequence~$\hat{\brmA}_1, \hat{\brmA}_2, \ldots$, matrix~$\brmB$, weighting~$c$
\Require Unbiased estimate of the $\kappa$-smoothed spectral density $\tilde{\psi(\lambda)}$
\State $\bpi \gets \left[ |\gamma_0(\lambda,\kappa)|, |\gamma_1(\lambda, \kappa)|, |\gamma_2(\lambda,\kappa)|, \ldots \right] / \sum_{k=0}^\infty |\gamma_k(\lambda, \kappa)|$
\State $k \sim \bpi$
\State $\brmx \sim \distNorm(0, \identity)$
\State $\hat{\brmf}_{-2} \gets \brmx$
\State $\hat{\brmf}_{-1} \gets \hat{\brmA}_1\brmx$ 
\For{$k'\gets 2$ to $k$}
\State $\hat{\brmf} \gets 2\hat{\brmA}_{k'}\hat{\brmf}_{-1} - \hat{\brmf}_{-2}$
\State $\hat{\brmf}_{-2} \gets \hat{\brmf}_{-1}$
\State $\hat{\brmf}_{-1} \gets \hat{\brmf}$
\EndFor
\State \Return $\frac{\gamma_k(\lambda,\kappa)}{\pi_k}\left(\brmx^{\trans}\hat{\brmf}
- c \cdot \left(\brmx^{\trans}\brmB\brmx - \trace(\brmB)\right)
\right)$
\end{algorithmic}
\label{alg:cv}
\end{algorithm}

\begin{lemma}
When~${\brmx\sim\distNorm(0,\identity)}$, the variance-minimizing weighting~$c^\star$ of the control variate in the expression for~$\hat{Z}'$ in Equation~\eqref{eq:control-variate} is~$\trace({\brmA\brmB})/\trace({\brmB^2})$.
\end{lemma}
\begin{proof}
For a estimator~$\hat{Z}'$ formed from~$\hat{Z}$ and a zero-mean control variate~$\hat{W}$, a classic result from Monte Carlo is that the variance-minimizing weighting is given by~${c^\star = \mathbb{C}[\hat{Z},\hat{W}]/\mathbb{V}[\hat{W}]}$, where~$\mathbb{C}[\hat{Z},\hat{W}]$ is the covariance between the original estimator and the control-variate.  In the specific case of interest here, then:
\begin{align}
c^\star &= \frac{\mathbb{C}(\brmx^{\trans}\brmA\brmx,\brmx^{\trans}\brmB\brmx - \trace(\brmB))}{\mathbb{V}(\brmx^{\trans}\brmB\brmx - \trace(\brmB))}\,,
\end{align}
where~$\mathbb{C}(\cdot,\cdot)$ is the covariance between the estimator and the control variate under the distribution induced by~$\brmx$.  The denominator is~$2\,\trace(\brmB^2)$ using the previous lemma.  The covariance in the numerator can be computed via
\begin{align}
\mathbb{C}(\brmx^{\trans}\brmA\brmx,\brmx^{\trans}\brmB\brmx - \trace(\brmB))
&= \mathbb{E}[\brmx^{\trans}\brmA\brmx(\brmx^{\trans}\brmB\brmx - \trace(\brmB))]
- \mathbb{E}[\brmx^{\trans}\brmA\brmx]\mathbb{E}[\brmx^{\trans}\brmB\brmx - \trace(\brmB)]\\
&= \mathbb{E}[\brmx^{\trans}\brmA\brmx(\brmx^{\trans}\brmB\brmx - \trace(\brmB))]\\
&= \mathbb{E}[\brmx^{\trans}\brmA\brmx\brmx^{\trans}\brmB\brmx] - \trace(\brmA)\trace(\brmB)\\
&= \trace(\mathbb{E}[\brmx\brmx^{\trans}\brmA\brmx\brmx^{\trans}]\brmB) - \trace(\brmA)\trace(\brmB)\\
&= \trace((2\brmA + \trace(\brmA)\identity)\brmB) - \trace(\brmA)\trace(\brmB)\\
&= 2\,\trace(\brmA\brmB)\,.
\end{align}
\end{proof}
For a given~$\brmB$ and weighting~$c$, some algebra shows that the variance is reduced by
\begin{align}
\mathbb{V}[\brmx^{\trans}\brmA\brmx] - \mathbb{V}[\brmx^{\trans}\brmA\brmx - c \cdot (\brmx^{\trans}\brmB\brmx - \trace(\brmB))]
%&= 2\,c\cdot \mathbb{C}[\brmx^{\trans}\brmA\brmx, \brmx^{\trans}\brmB\brmx - \trace(\brmB)]
%- c^2\cdot \mathbb{V}[\brmx^{\trans}\brmB\brmx - \trace(\brmB)]\\
&= 4\,c\cdot\trace(\brmA\brmB) - 2\,c^2\cdot\trace(\brmB^2)\,.
\end{align}
When~$c$ is chosen optimally, the control variate reduces the variance by~$2\,\trace(\brmA\brmB)^2/\trace(\brmB^2)$.  This is maximized in the trivial case when~${\brmB=\brmA}$.  The first case of practical interest is~${\brmB=\alpha\identity}$ for scalar~${\alpha\neq 0}$.  This reduces the variance by an amount that does not depend on~$\alpha$, as the optimal~$c$ scales the control variate appropriately.  The second case of practical interest is where~$\brmB$ is a diagonal matrix and, unsurprisingly, the variance minimizing diagonal matrix is the one that is the diagonal of~$\brmA$.  If the Skilling-Hutchinson trace estimator is used sequentially, then the raw diagonal estimator sitting beneath it can be used to form a diagonal~$\brmB$, i.e.,~${\mathbb{E}[\brmx^{\trans}\brmA \odot \brmx]=\text{diag}(\brmA)}$ where~$\odot$ is element-wise product \citep{bekas2007estimator}.  That is, the diagonal estimate from previous estimations can be used to form the control variate for later estimations.  We also note that the control variate described above can be easily adapted for variance reduction in the diagonal estimator.  When estimating~$\brmB$ to be the diagonal of~$\brmA$, we set~${c=1}$.  Pseudocode is provided in Algorithm~\ref{alg:cv}.

%% file: evaluation.tex
\section{Empirical Evaluation of Estimator Performance}
In this section we evaluate the performance of the estimator and demonstrate the feasibility of approximating the spectra of large matrices. We focus on two classes of matrices: sparse adjacency matrices of certain undirected graphs known as Kneser graphs, and dense random matrices designed to model the Hessian matrices often appearing in machine learning problems.
\subsection{Kneser Graphs}
\label{sec:kneser}

\begin{figure}[t]
\centering
\includegraphics[width=\textwidth]{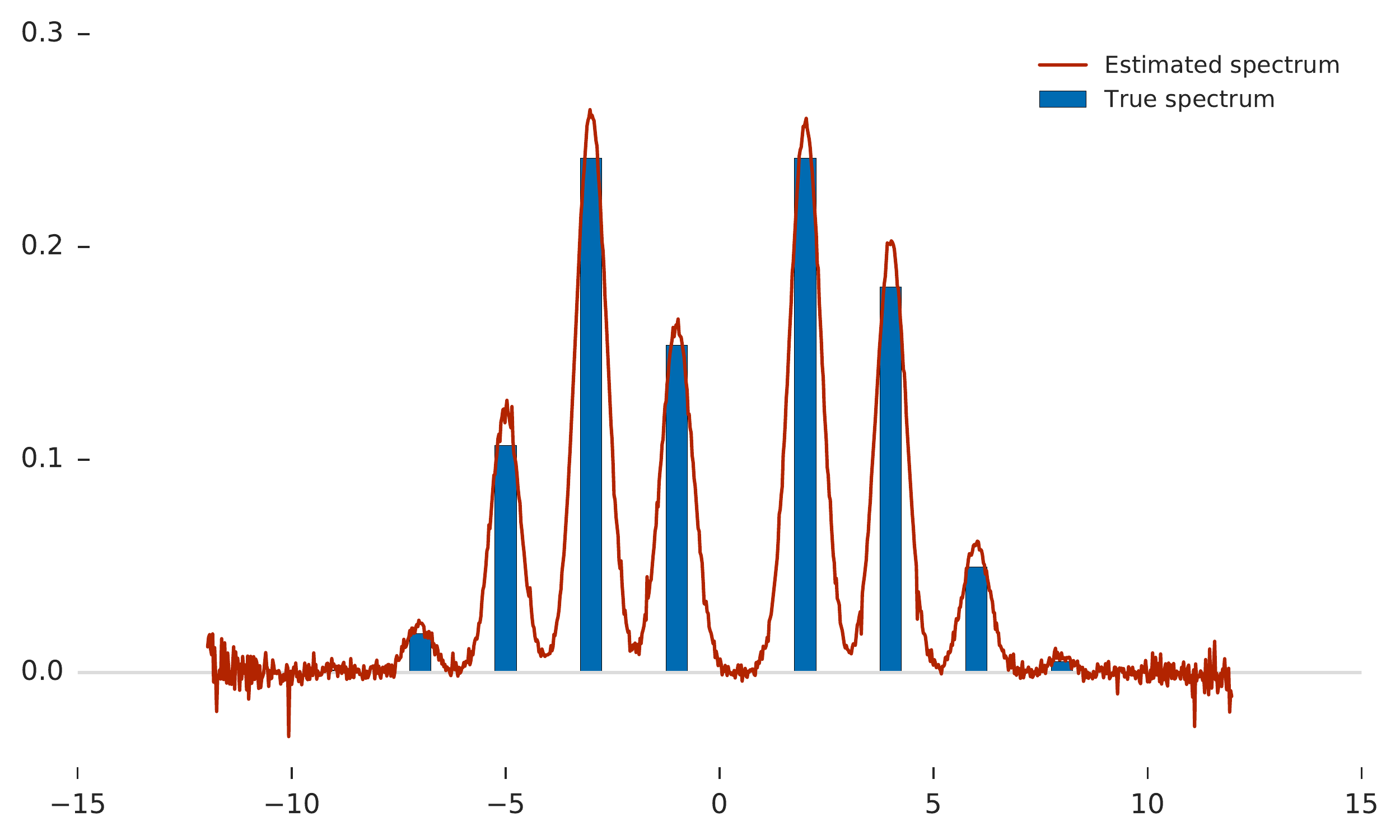}
\caption{The estimated and true spectrum of the Kneser graph $K(23, 11)$. This estimate was generated with $\kappa = 1000$.}
\label{fig:kneser}
\end{figure}

Graph theory provides a natural source of large matrices with well-understood spectra. The \emph{Kneser Graph}, $K(n, k)$ is constructed from a set $S$ with $|S| = n$. The vertices of $K(n, k)$ are identified with the subsets of $S$ of size $k$:
$$V = \{R: R\subset S, |R| = k\}$$
Two vertices are adjacent if they are disjoint:
$$E = \{(R, R'): R\cap R' = \varnothing\}$$
The graph $K(n, k)$ is regular, and each vertex has degree $\binom{n - k}{k}$. The eigenvalues of $K(n, k)$ are 
$$\lambda_i = (-1)^i\binom{n - k - i}{k - i}$$
with associated multiplicities
$$m_i = \binom{n}{i} - \binom{n}{i - 1}$$
for $i = 0,...,k$. For a more detailed discussion of Kneser graphs, see \citet{godsil2001algebraic}.

Choosing $n = 23$ and $k = 11$, we get a regular graph with $1352078$ and $16224936$ edges and degree $12$. This graph has the following eigenvalues, with associated multiplicities:

\begin{center}
\begin{tabular}{ c c }
 Eigenvalue & Multiplicity \\ 
 -11 & 22 \\
-9 & 1518 \\
-7 & 24794 \\
-5 & 144210 \\
-3 & 326876 \\
-1 & 208012 \\
2 & 326876 \\
4 & 245157 \\
6 & 67298 \\
8 & 7084 \\
10 & 230 \\
12 & 1 \\
\end{tabular}
\end{center}

To better simulate the situation where we have an unbiased estimator with non-zero variance, we added Gaussian noise with variance $1 / d^2$ to each non-zero entry of the adjacency matrix for $K(23, 11)$. This noise was added independently each time matrix multiplication was applied. The smoothed spectral density was estimated as described in Algorithm~\ref{alg:basic}. The one exception is that, for each random vector $\brmx$, we select 10000 indices $k\sim\bpi$. We obtained excellent agreement between our predictions and the known eigenspectrum of $K(23, 11)$, as shown in Figure~\ref{fig:kneser}.

\subsection{Random Matrix Theory}
A random matrix $\brmA$ is a matrix whose entries $\rmA_{ij}$ are random variables. The distribution of the~$\rmA_{ij}$ determines the type of random matrix ensemble to which $\brmA$ belongs. We will focus on two ensembles: the real \emph{Wigner} ensemble for which $\brmA$ is symmetric but otherwise the $\rmA_{ij}$ are independent with mean zero; and the real \emph{Wishart} ensemble for which $\brmA=\brmB\brmB^{\trans}$, where the columns of $\brmB$ are independent and normally distributed with mean zero. We will also focus on the \emph{limiting spectral density} of $\brmA$, i.e., the large $D$ limit of the empirical spectral distribution defined in Equation.~\ref{eqn:discrete-density},
\begin{equation}
\psi_\infty(\lambda)  \equiv \lim_{D\to\infty} \psi_D(\lambda)  = \lim_{D\to\infty}  \frac{1}{D} \sum_{d=1}^D \delta({\lambda - \lambda_d(\brmA)})\,.
\end{equation}
Here we attach the~$D$ subscript to~$\psi$ to indicate the size of the matrix when it is relevant, with~$\psi_\infty(\lambda)$ being the limiting spectral density.
For the random matrix ensembles considered here, and under mild conditions (such as bounded fourth moments), $\psi_D(\lambda)$ converges almost surely to $\psi_\infty(\lambda)$.

Throughout this section, validation calculations were performed with noisy versions of the Wishart and Wigner ensembles. Multiplicative Gaussian noise with variance $100 / d^2$ was applied independently at each matrix multiplication step, where $d$ is the dimension of the matrix in question. As in Subsection~\ref{sec:kneser}, spectral density was estimated as described in Algorithm~\ref{alg:basic}. This time, we select 100000 indices $k\sim\bpi$
for each of 2048 random vectors $\brmx$.

\subsubsection{Wigner Ensemble}
For concreteness, we consider a real Wigner matrix $\brmA$ whose elements are normally distributed,
\begin{eqnarray}
\rmA_{ij} \sim
	\begin{cases}
		\mathcal{N}(0,\frac{1}{4D}) & \text{if} \  i < j\\
		\mathcal{N}(0,\frac{1}{2D}) & \text{if} \ i = j \\
		\rmA_{ji} & \text{otherwise}
	\end{cases}\, .
\end{eqnarray}
The limiting spectral density of a random matrix of this type was first computed by~\citet{Wigner}. The result is the celebrated semi-circle law,
\begin{equation}
\label{eqn:rho_sc}
\psi_\infty(\lambda) = \psi_{\sf{sc}}(\lambda) = \begin{cases}
  \frac{2}{\pi}\sqrt{1 - \lambda^2} & \text{if}\ |\lambda| \le 1 \\
  0 & \text{otherwise}
 \end{cases}\,.
\end{equation}
Because $\psi_{\sf{sc}}$ is an even function, its odd moments vanish. The even moments are related to the Catalan numbers, $C_k$,
\begin{equation}
\begin{split}
m_{2k}  &= \lim_{D\to\infty}\frac{1}{D}\trace(\brmA^{2k}) = \int_{-\infty}^{\infty} \lambda^{2k} \, \psi_{\sf{sc}}(\lambda) \, d\lambda \\
&= 4^{-k}\,C_{k} = 4^{-k} \frac{1}{k+1}\binom{2k}{k}\,.
\end{split}
\end{equation}
From Stirling's formula, it is easy to see that for large $k$, the moments decrease as the $3/2$ power,
\begin{equation}
m_{2k} = \frac{1}{\sqrt{\pi} k^{3/2}} + \mathcal{O}(k^{-5/2})\,,
\end{equation}
but they are nonzero for any finite $k$. In contrast, the Chebyshev expansion actually truncates at finite order. Indeed, by inspection we have,
\begin{equation}
\psi_{\sf{sc}}(\lambda) = \frac{T_0(\lambda) - T_2(\lambda)}{\pi \sqrt{1-\lambda^2}}\, ,
\end{equation}
so that,
\begin{equation}
\begin{split}
\lim_{D\to\infty}\frac{1}{D} \trace(\brmT_k) &= \int_{-\infty}^{\infty} T_k(\lambda) \, \psi_{\sf{sc}}(\lambda) \, d\lambda \\
&=\int_{-\infty}^{\infty} T_k(\lambda) \left(\frac{T_0(\lambda) - T_2(\lambda)}{\pi \sqrt{1-\lambda^2}}\right) \, d\lambda\\
&= \delta_{0,k} - \frac{1}{2} \delta_{2,k}\,.
\end{split}
\end{equation}
This result holds only as $D\to\infty$; there will be finite-size corrections to this result for finite~$D$. Moreover, for our simulations, there is additional error resulting from the variance of our estimators; nevertheless, we still observe excellent agreement with these infinite-$D$ limiting predictions. See Figure~\ref{fig:wishwig}.

\begin{figure}[t!]
\centering
\begin{subfigure}[b]{0.49\textwidth}
  \includegraphics[width=\textwidth]{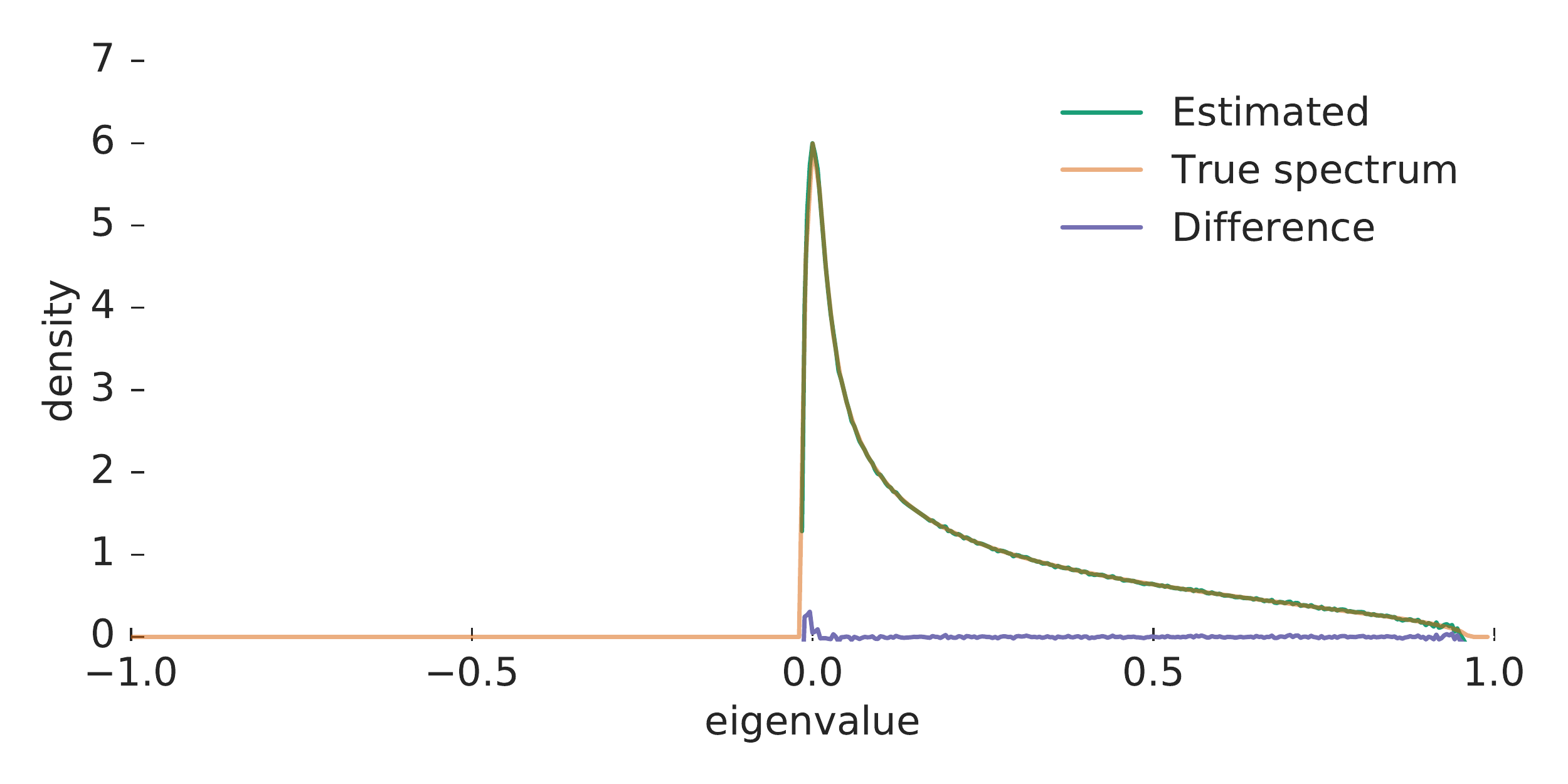}%
  \caption{$0.04$ Wigner + $0.96$ Wishart}%
\end{subfigure}~%
\begin{subfigure}[b]{0.49\textwidth}
  \includegraphics[width=\textwidth]{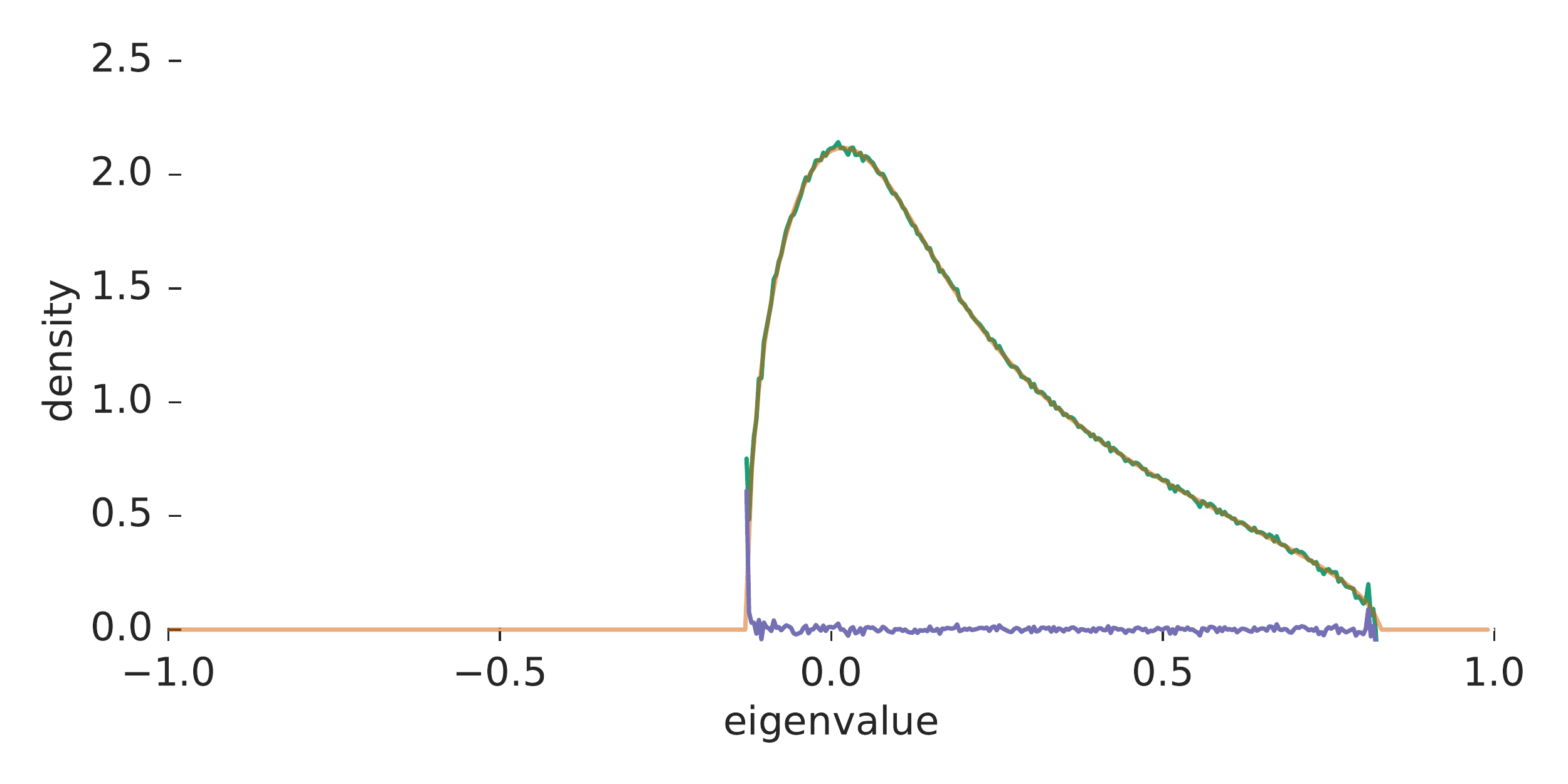}%
  \caption{$0.2$ Wigner + $0.8$ Wishart}%
\end{subfigure}\\
\begin{subfigure}[b]{0.49\textwidth}
  \includegraphics[width=\textwidth]{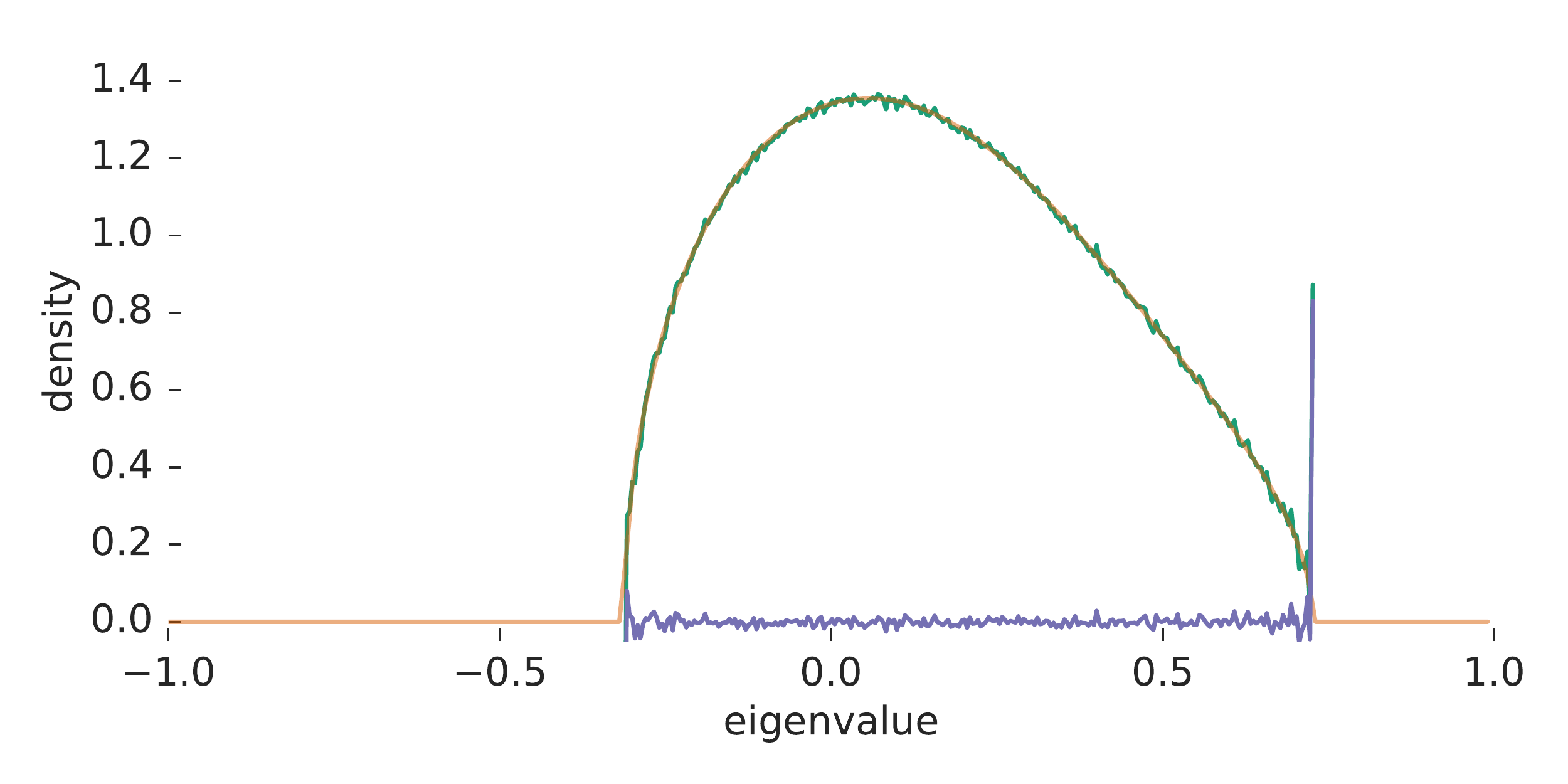}%
  \caption{$0.4$ Wigner + $0.6$ Wishart}%
\end{subfigure}~%
\begin{subfigure}[b]{0.49\textwidth}
  \includegraphics[width=\textwidth]{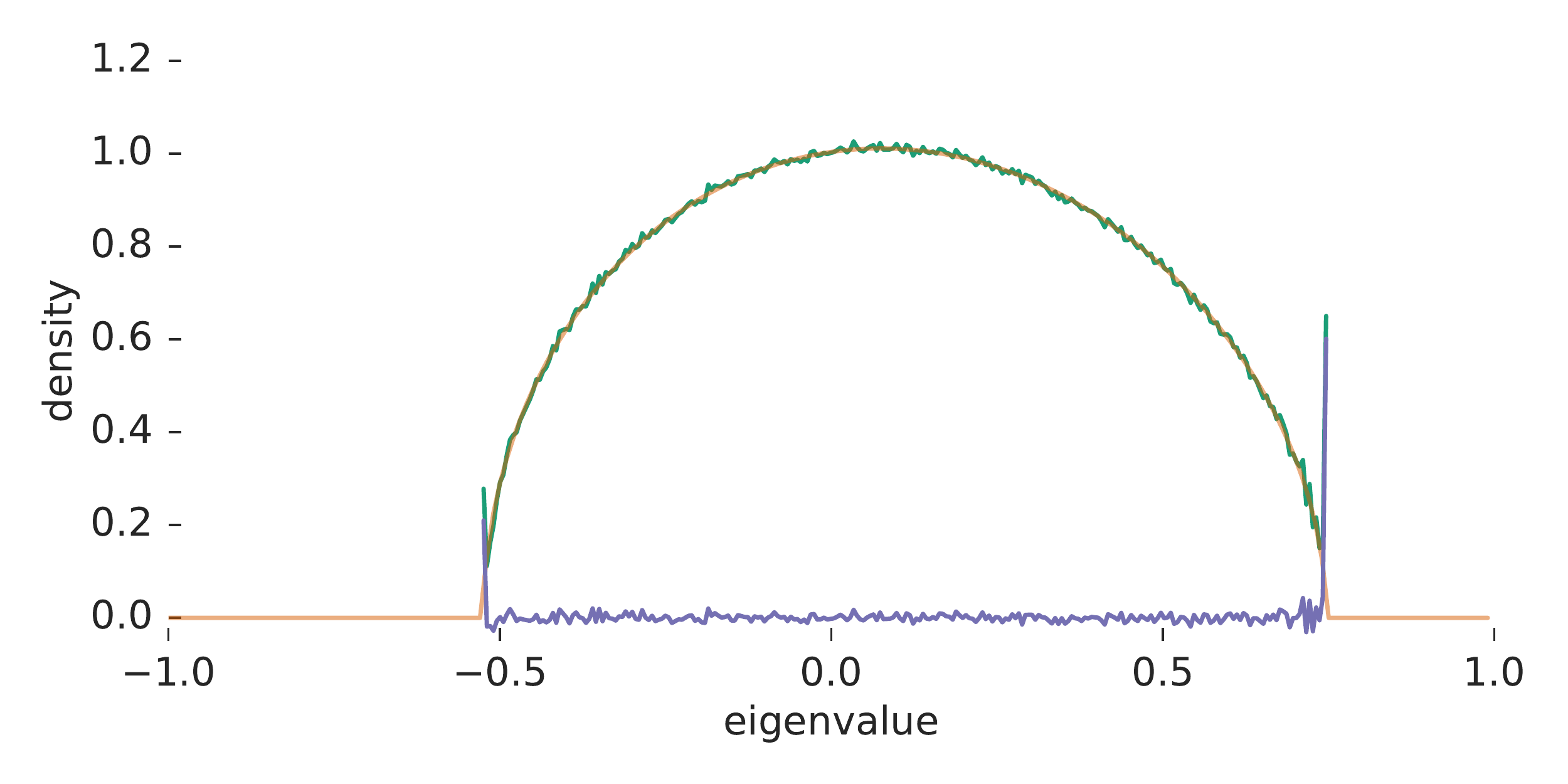}%
  \caption{$0.6$ Wigner + $0.4$ Wishart}%
\end{subfigure}\\
\begin{subfigure}[b]{0.49\textwidth}
  \includegraphics[width=\textwidth]{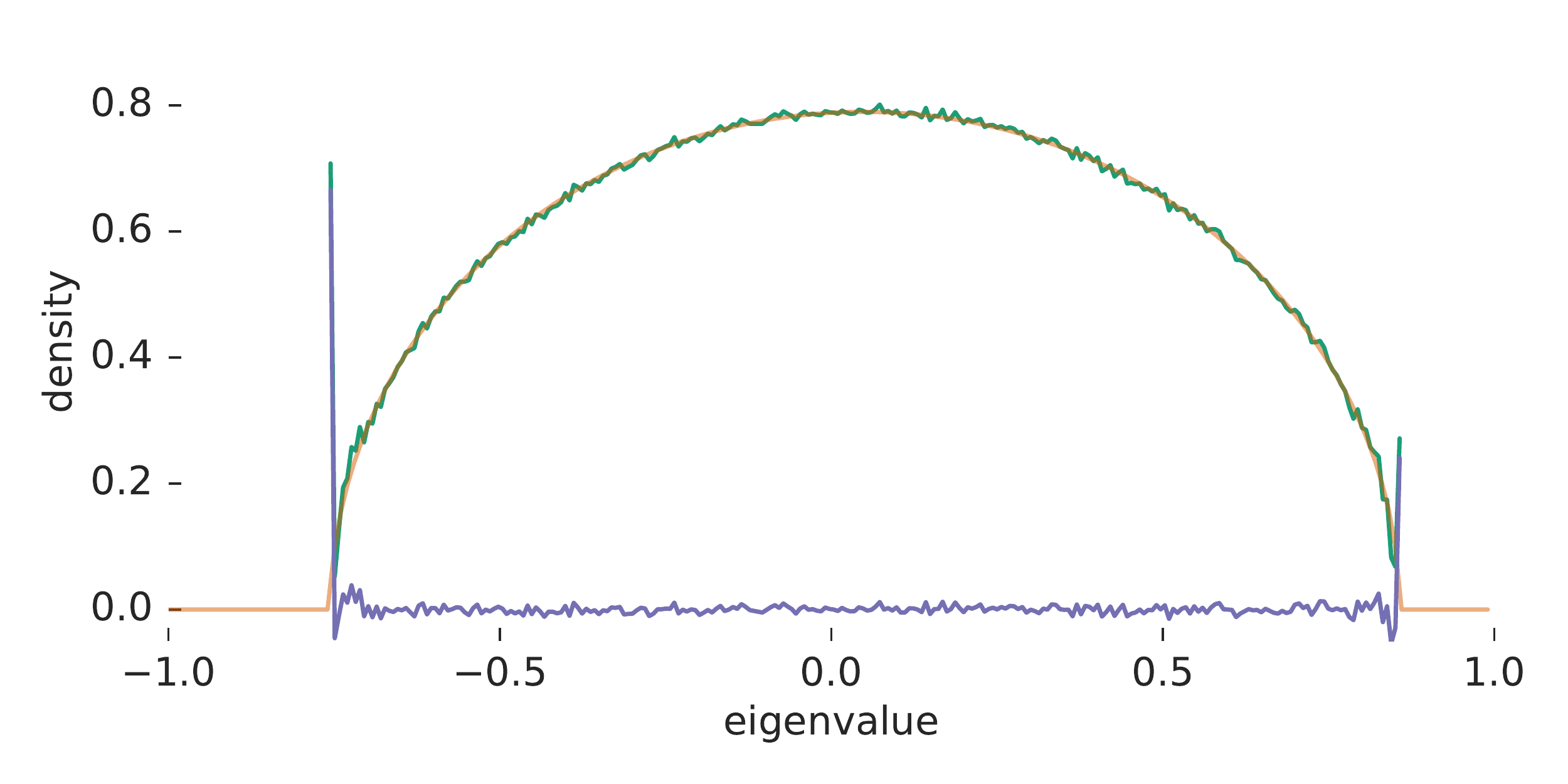}%
  \caption{$0.8$ Wigner + $0.2$ Wishart}%
\end{subfigure}~%
\begin{subfigure}[b]{0.49\textwidth}
  \includegraphics[width=\textwidth]{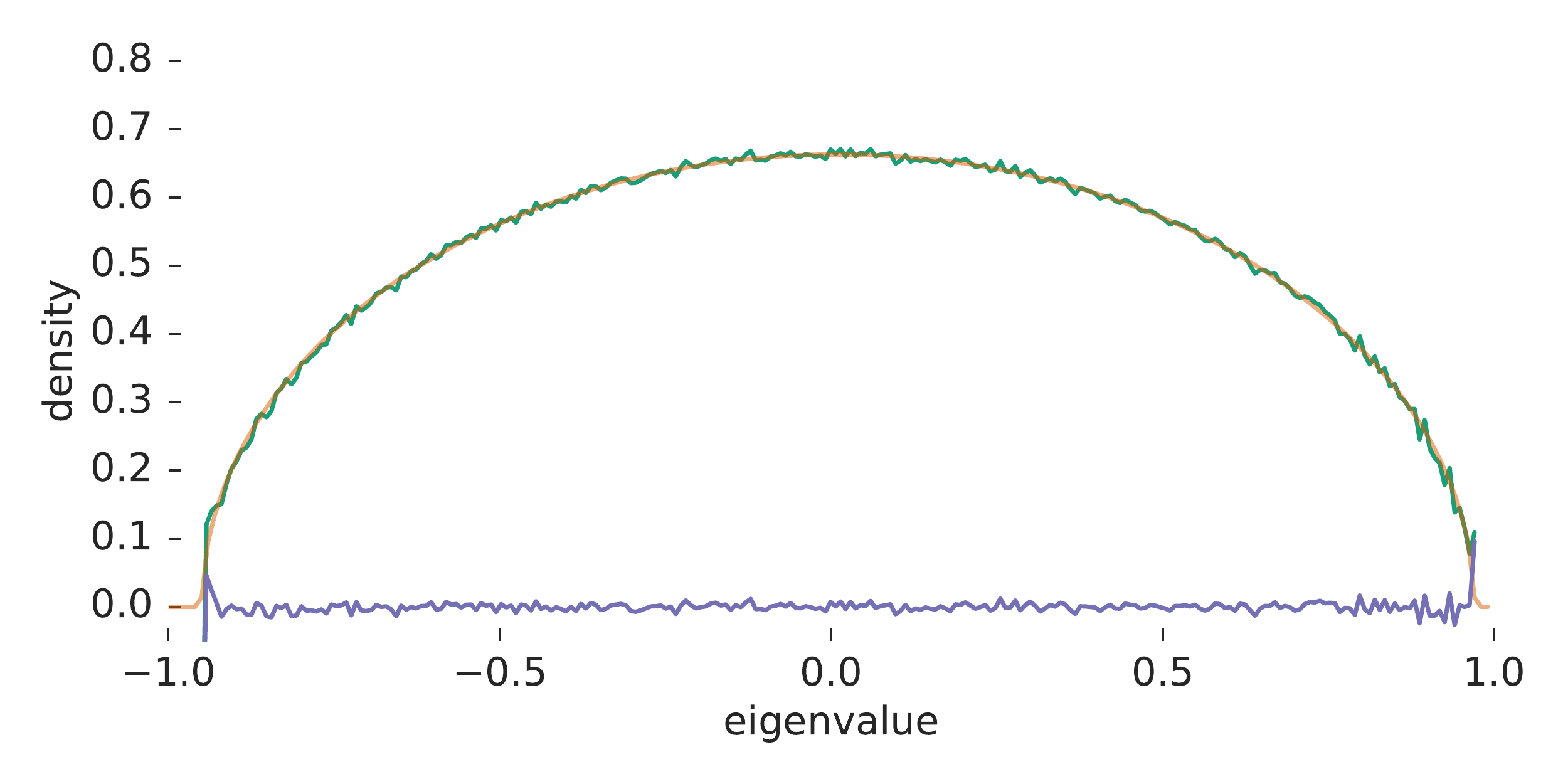}%
  \caption{$0.96$ Wigner + $0.04$ Wishart}%
\end{subfigure}
\caption{The estimated and true spectral densities of six matrices are shown.  These matrices are convex sums of Wishart and Wigner matrices, so their spectra are known in closed form to provide ground truth for the randomized estimation. Estimates were produced using 500 samples with~${\kappa = 5000}$.  Although some estimates exhibit variance due to boundary effects, the overall differences are small.}
\label{fig:wishwig}
\end{figure}

\subsubsection{Wishart Ensemble}
Next we consider matrices~$\brmA$ drawn from the Wishart ensemble. In particular, we consider~${\brmA=\brmB\brmB^{\trans}}$ where~${\brmB\in\mathbb{R}^{D\times N}}$ and~${\rmB_{ij} \sim \mathcal{N}(0,\frac{\sigma^2}{N})}$. If both~$N$ and~$D$ are taken to infinity at the same rate, i.e.,~${D,N \to \infty}$ with~${D = \phi N}$ for some constant~$\phi$, then the limiting spectral density can be shown to obey the Marchenko-Pastur distribution~\citep{MarchenkoPastur},
\begin{equation}
\label{eqn:rho_mp}
\psi_{\sf{mp}}(\lambda) = \begin{cases}
  \psi(\lambda) & \text{if}\ \phi < 1 \\
  (1-\phi^{-1})\delta(\lambda)+ \psi(\lambda) & \text{otherwise}
 \end{cases}\,,
\end{equation}
where,
\begin{equation}
\label{eqn:mp_gap}
\psi(\lambda) = \frac{1}{2\pi \lambda \sigma^2 \phi} \sqrt{(\lambda - \lambda_{-})(\lambda_{+} - \lambda)}\;\indicator_{\lambda_-\le \lambda\le \lambda_+}\quad\text{and}\quad\lambda_{\pm} = \sigma^2 (1\pm\sqrt{\phi})^2 \,.
\end{equation}
The Wishart matrix $\brmA$ is low rank if $\phi > 1$, and the nullity equals~${(N-D)/N = 1-\phi}$, which is the source of the delta function density at $0$ in that case. Note that the minimum eigenvalue is $\lambda_-$, which may be larger than $0$, and depends on the value of $\phi$. The moments of the Marchnko-Pastur distribution are polynomials in $\phi$ with coefficients equal to the Narayana numbers,
\begin{equation}
\begin{split}
m_{k}  &= \lim_{D\to\infty}\frac{1}{D}\trace(\brmA^{k}) = \int_{-\infty}^{\infty} \lambda^{k} \, \psi_{\sf{mp}}(\lambda) d\lambda \\
&= \sigma^{k} \sum_{j=0}^{k-1} N(k,j+1)\,\phi^j\\
&= \sigma^{k}\, {}_2F_1(1-k,-k,2,\phi)\,.
\end{split}
\end{equation}
where $_2F_1$ is the Gauss hypergeometric function and $N(j,k)$ are the Narayana numbers,
\begin{equation}
N(k,j) = \frac{1}{k}\binom{k}{j}\binom{k}{j-1}\,.
\end{equation}
Although relatively complicated, the Chebyshev coefficients can be written as a sum over the $m_k$,
\begin{equation}
\begin{split}
\lim_{D\to\infty}\frac{1}{D} \trace(\brmT_k) &= \int_{-\infty}^{\infty} T_k(\lambda) \, \psi_{\sf{mp}}(\lambda) \, d\lambda \\
&= \frac{k}{2} \sum_{j=0}^{\left \lfloor{k/2}\right \rfloor} \frac{(-1)^j}{k-j}\binom{k-j}{j}2^{k-2j}m_{k-2j}\\
&= \frac{k}{2} \sum_{j=0}^{\left \lfloor{k/2}\right \rfloor} \frac{(-1)^j}{k-j}\binom{k-j}{j}(2\sigma)^{k-2j}\, {}_2F_1(1-k+2j,2j-k,2,\phi)\,.
\end{split}
\end{equation}
One reason that this expansion is more complicated than that of the Wigner case is because the support of the spectrum is in the range $[\lambda_-, \lambda_+]$, whereas the Chebyshev polynomials are a natural basis on $[-1,1]$. Indeed, a much simpler expansion arises when $\brmA$ is scaled and shifted so that its spectrum has support on $[-1,1]$. For $\phi<1$, letting $\bm{\mathrm{\tilde{A}}}  = \alpha \brmA + \beta \brmI_D$ with $\alpha = \frac{1}{2\sigma^2 \sqrt{\phi}}$ and $\beta = - \frac{1+\phi}{2\sqrt{\phi}}$ gives a spectral density $\tilde{\psi}_{\sf{mp}}(\lambda)$,
\begin{equation}
\tilde{\psi}_{\sf{mp}}(\lambda) = \frac{1}{\alpha}\psi_{\sf{mp}}(\frac{\lambda - \beta}{\alpha}) = \frac{2}{\pi}\frac{\sqrt{1-\lambda^2}}{1+2\lambda \sqrt{\phi}+\phi}\,.
\end{equation}
The Chebyshev coefficients can be found to equal,
\begin{equation}
\lim_{D\to\infty}\frac{1}{D} \trace(\bm{\mathrm{\tilde{A}}}_k) = \begin{cases}
\frac{\pi}{2} & k = 0\\
-\frac{\pi}{4}\sqrt{\phi} & k = 1 \\
-\frac{\pi}{4}(1-\phi)(-\sqrt{\phi})^{k-2} & k \ge 2
\end{cases} .
\end{equation}
Again, these results hold only as $D\to\infty$, and again there is still additional variance from our estimators, but we still observe excellent agreement with these infinite-$D$ limiting predictions. See Figure~\ref{fig:wishwig}.

\subsubsection{Wishart plus Wigner}
As described by~\citet{pennington2017geometry}, the sum of a Wishart and a Wigner matrix can serve as a simplified model of the Hessian matrix of the loss function of machine learning models, as we now review. Suppose we consider an arbitrary (possibly non-convex) parameterizable model function $f(\theta) \equiv f_\theta(x)$ of some data $x$. If the model is evaluated with a convex loss function $l(f(\theta))$, then the Hessian matrix has a simple form. In particular, the second derivative of the loss with respect to the parameters $\theta$ is,
\begin{equation}
\begin{split}
\rmH_{ij} &= \frac{\partial^2 l(f(\theta))}{\partial \theta_i \partial \theta_j} = \frac{\partial}{\partial \theta_i} \frac{\partial l(f(\theta))}{\partial f^\alpha(\theta)} \frac{\partial f^\alpha(\theta)}{\partial \theta_j} =  \frac{\partial^2 l(f(\theta))}{\partial f^\alpha(\theta)\partial f^\beta(\theta)}\frac{\partial f^\alpha(\theta)}{\partial \theta_j}\frac{\partial f^\beta(\theta)}{\partial \theta_i} + \frac{\partial l(f(\theta))}{\partial f^\alpha(\theta)} \frac{\partial^2 f^\alpha(\theta)}{\partial \theta_i\partial \theta_j}\\
&= [\brmH_{l}]_{ij} + [\brmH_{f}]_{ij}\,,
\end{split}
\end{equation}
where,
\begin{equation}
[\brmH_l ]_{ij}= \rmJ_{i\alpha} \rmJ_{j\beta}  \frac{\partial^2 l(f(\theta))}{\partial f^\alpha(\theta)\partial f^\beta(\theta)}\,,\quad\text{and}\quad [\rmH_f]_{ij} = \frac{\partial l(f(\theta))}{\partial f^\alpha(\theta)} \frac{\partial^2 f^\alpha(\theta)}{\partial \theta_i\partial \theta_j}\,,
\end{equation}
and the Jacobian matrix is,
\begin{equation}
\rmJ_{i\alpha} = \frac{\partial f^\alpha(\theta)}{\partial \theta_i} \,.
\end{equation}
Because we've assumed $l$ to be convex, $\brmH_l$ is PSD. A simple approximation which captures this additive structure of the Hessian is outlined in~\cite{pennington2017geometry}, and involves approximating $\brmH_l$ as a Wishart matrix and $\brmH_f$ as a Wigner matrix. In particular, we will consider the approximation,
\begin{equation}
\brmH = \gamma\brmB  + (1-\gamma) \brmC\,,
\end{equation}
where $\brmB$ is standard Wigner and $\brmC$ is standard Wishart, as defined above, and $\gamma$ is an interpolation parameter that can be interpreted as a measure of the loss value at convergence. When~${\gamma=1}$, the model has converged to a high-index saddle point with a large loss value, at which point the Hessian has no contribution from the Wishart matrix $\brmC$ and has an equal number of positive and negative eigenvalues. The global minimum corresponds to ${\gamma = 0}$, at which point there is no contribution from the Wigner matrix $\brmB$ and the Hessian has no negative eigenvalues.  As described in~\cite{pennington2017geometry}, the number of negative eigenvalues can be computed as a function of $\epsilon \equiv \frac{\gamma^2}{2(1-\gamma)^2}$, which measures the relative contribution of $\brmB$ and $\brmC$ to $\brmH$. The details are not particularly illuminating, but the result is,
\begin{equation}
\alpha(\epsilon) = \frac{1}{\pi}\left[\frac{\xi_{-}(\xi_{+}-2\sqrt{2}\xi \epsilon)}{12\sqrt{3}\xi^2 \phi^2 \epsilon} + \arctan\left(\frac{\sqrt{3}\xi_{-}}{\xi_{+}-2\sqrt{2}\xi\epsilon}\right) - \frac{1}{\phi}\arctan\left(\frac{\sqrt{3}\xi_{-}}{\xi_{+}+4\sqrt{2}\xi\epsilon}\right) \right]_+\,,
\end{equation}
where $\xi_{\pm} = \xi^2 \pm \chi_2$ and,
\begin{equation}
\xi = 2^{-1/6} \left(\chi_1 + \sqrt{\chi_1^2 -2 \chi_2^3}\right)^{1/3}\,,\quad \chi_1 = 4\epsilon^3 + 9 \epsilon^2 \phi + 18 \epsilon^2 \phi^2\,,\qquad \chi_2 = 2\epsilon^2 + 3\epsilon\phi - 3\epsilon \phi^2\,.
\end{equation}
Figure~\ref{fig:wishwig-index} compares the empirical estimation with these analytic predictions.

\begin{figure}[t!]
\centering
\includegraphics[width=0.8\textwidth]{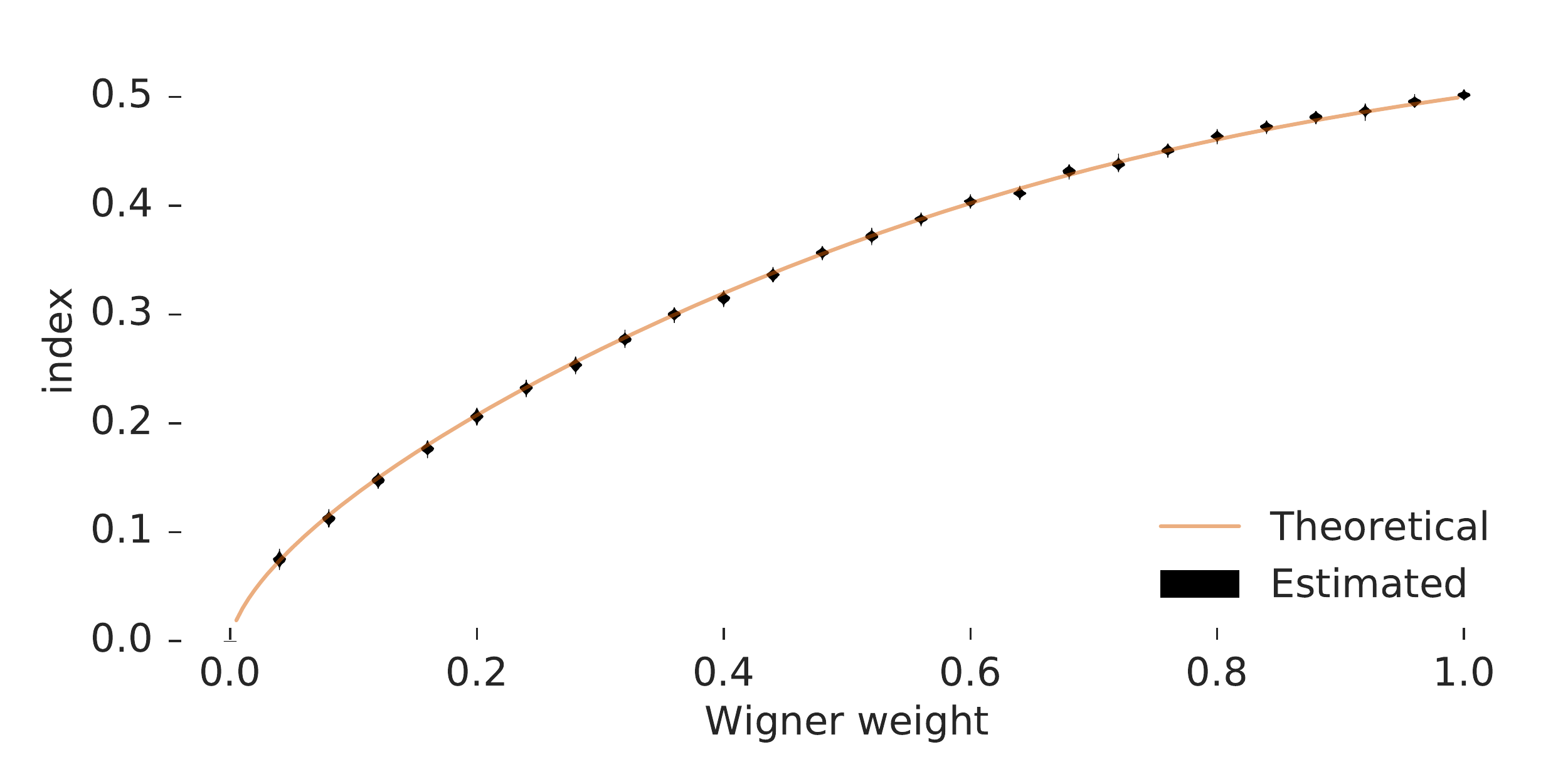}%
\caption{An example of computing a useful statistic of a large matrix via its spectral density.  Here is plotted theoretical and empirical estimates of the index (fraction of negative eigenvalues) of a Wishart+Wigner matrix, as a function of the Wigner fraction.  The empirical data here are violin plots computed via the bootstrap over 100 Monte Carlo estimates of the mean spectral density.}
\label{fig:wishwig-index}
\end{figure}

%% file: conclusion.tex
\section{Discussion}
In this work we have described a framework for estimation of spectral densities of large matrices.  Our primary objective has been to construct a tool for the empirical investigation of matrices that can only be probed in implicit and noisy ways.  Our overall approach has been to combine ideas that have been proposed across several different fields into an estimation framework that provides unbiased estimates with controlled variance, even when the matrices become large.  We use this framework to then construct smoothed estimates of the overall density.  This smoothing necessarily introduces bias, but it allows us to visualize spectra whose atomic nature would otherwise make them difficult to interrogate.  In future work, we intend to use this tool in the analysis of high dimensional statistical and machine learning models, where the local geometry of the optimization surface---as captured by the Hessian and Fisher matrices---is challenging to examine.